\newif\ifextendedv
\newif\iffuturev
\newif\ifextendstaggered
\newif\ifijrrv
\newif\iftrov
\extendedvtrue
\futurevtrue
\extendstaggeredfalse
\ijrrvfalse
\trovtrue
\ifijrrv
\documentclass[Afour,sageh,times,doublespace]{sagej}
\usepackage{moreverb,url}
\usepackage[dvipsnames]{xcolor}

\usepackage[colorlinks,bookmarksopen,bookmarksnumbered,citecolor=red,urlcolor=red]{hyperref}

\newcommand\BibTeX{{\rmfamily B\kern-.05em \textsc{i\kern-.025em b}\kern-.08em
T\kern-.1667em\lower.7ex\hbox{E}\kern-.125emX}}

\else
\iftrov
\documentclass[journal]{IEEEtran} 
\else
\documentclass[letterpaper, 10 pt, conference]{ieeeconf}
\fi

\IEEEoverridecommandlockouts               

\usepackage[colorlinks=true, bookmarks=true, citecolor=blue, backref=true]{hyperref}
\usepackage[pdftex]{graphicx}

\usepackage[noadjust]{cite}

\makeatletter
\let\NAT@parse\undefined
\makeatother
 
\usepackage[utf8]{inputenc}
\usepackage{amsthm}
\usepackage{amssymb,amsfonts}
\usepackage[noend]{algpseudocode}

\fi

\usepackage{comment}
\usepackage{xspace}
\usepackage{xcolor}
\usepackage{algorithm,algcompatible}

\algnewcommand\algorithmicto{\textbf{to}}
\algnewcommand\RETURN{\State \textbf{return} }
\usepackage{sidecap}
\usepackage{bm}

\usepackage{mathtools}

\usepackage[font=small]{caption}

\usepackage{subcaption}
\usepackage{float}
\let\labelindent\relax
\usepackage{enumitem}
\usepackage{multirow}

\newif\ifrelaxedv  
\relaxedvtrue

\ifrelaxedv \else
\renewcommand{\baselinestretch}{0.93}
\fi

 \def\C{\mathcal{C}} 
  \def\L{\mathcal{L}}

 \def\L{\mathcal{L}} 
  
  \def\M{\mathcal{M}}
\def\X{\mathcal{X}}  \def\Y{\mathcal{Y}}
\def\L{\mathcal{L}} 

  \def\dC{\mathbb{C}}
  
 \def\dN{\mathbb{N}} 
\def\dR{\mathbb{R}} \def\dM{\mathbb{M}} 
  \def\dM{\mathbb{M}}

\def\eps{\varepsilon}

\def\CH{{\textrm{CH}}}
\def\OPT{{\textrm{OPT}}}
\def\cost{{\textrm{cost}}}

\newcommand{\prm}{{\tt PRM}\xspace}

\newcommand{\drrt}{{\tt dRRT}\xspace}
\newcommand{\rrt}{{\tt RRT}\xspace}
\newcommand{\drrtstar}{{\tt dRRT$^*$}\xspace}
\newcommand{\mccbs}{{\tt MC-CBS}\xspace}
\newcommand{\cbs}{{\tt CBS}\xspace}
\newcommand{\mstar}{{\tt M$^*$}\xspace}
\newcommand{\astar}{{\tt A$^*$}\xspace}

\newcommand{\abs}[1]{\left| #1 \right|}
\newcommand{\norm}[1]{\left\| #1 \right\|}

\newcommand{\bigbrace}[1]{\left\{ #1 \right\}}
\newcommand{\bigpar}[1]{\left( #1 \right)}

\newcommand{\fakeparagraph}[1]{\vspace{5pt} \noindent\textbf{#1}}

\newcommand{\xcover}[1]{#1-cover}

\theoremstyle{definition}
\newtheorem{definition}{Definition}

\newtheorem{theorem}{Theorem}
\newtheorem{lemma}[theorem]{Lemma}

\def\marrow{\marginpar[\hfill$\longrightarrow$]{$\longleftarrow$}}
\def\dror#1{{\color{orange}\textsc{(\textbf{Dror says}: }\marrow\textsf{#1})}}
\def\kiril#1{{\color{blue}\textsc{(\textbf{KS}: }\marrow\textsf{#1})}}

\usepackage[normalem]{ulem}

\usepackage{amsmath}

\DeclareMathOperator*{\argmin}{arg\,min}

\let\directcite\cite

\ifijrrv
\let\indirectcite\citep
\else
\let\indirectcite\cite
\fi

\title{\LARGE \bf{
Near-Optimal Multi-Robot Motion Planning with Finite Sampling}}

\ifijrrv
\author{Dror Dayan\affilnum{1}, Kiril Solovey\affilnum{2}, Marco Pavone\affilnum{3}, and Dan Halperin\affilnum{1}}

\affiliation{\affilnum{1}Blavatnik School of Computer Science, Tel-Aviv University, Israel.\\
\affilnum{2}Computer Science Department, Technion -- Israel Institute of Technology, Israel.\\
\affilnum{3}Aeronautics and Astronautics Department, Stanford University, CA~94305, USA.}

\corrauth{Dror Dayan, Blavatnik School of Computer Science
Tel-Aviv University,
Israel.}

\email{drordayan@mail.tau.ac.il, dror.d.dayan@gmail.com}

\else
\author{Dror Dayan$^{1}$, Kiril Solovey$^{2}$, Marco Pavone$^{3}$, and Dan Halperin$^{1}$
	\thanks{$^{1}$ D.\ Dayan and D.\ Halperin are with the Blavatnik School of Computer Science, Tel-Aviv University, Israel. Their work was supported in part by the Israel Science Foundation (grant no.~1736/19), by the US NSF/US-Israel BSF (grant no. 2019754), by the Isarel Ministry of Science and Technology (grant no.~103129), by the Blavatnik Computer Science Research Fund, and by the Yandex Machine Learning Initiative for Machine Learning at Tel Aviv University.}%
	\thanks{$^{2}$ K.\ Solovey is with the Faculty of Electrical and Computer Engineering, Technion -- Israel Institute of Technology, where he is a Ravitz Fellow.}
	\thanks{$^{3}$ M.\ Pavone is with the Aeronautics and Astronautics Department, Stanford University, CA~94305, USA. His work is supported in part by the Toyota Research Institute (TRI) and the Center for Automotive Research at Stanford (CARS).}%
}
\fi
\begin{document}

\ifijrrv
\runninghead{Dayan et al.}

\fi

\ifijrrv
\else
\maketitle
\fi

\ifijrrv
\else
\ifextendedv
\thispagestyle{plain}
\pagestyle{plain}
\fi
\fi

\ifijrrv

\begin{abstract}
An underlying structure in several sampling-based methods for continuous multi-robot motion planning (MRMP) is the \emph{tensor roadmap}, which emerges from combining multiple probabilistic roadmap (PRM) graphs constructed for the individual robots via a tensor product. We study the conditions under which the tensor roadmap encodes a near-optimal solution for MRMP---satisfying these conditions implies near optimality for a variety of popular planners, including dRRT*, and the discrete methods M* and conflict-based search, when applied to the continuous domain. 
We develop the first finite-sample analysis of this kind, which specifies the number of samples, their deterministic distribution, and magnitude of the connection radii that should be used by each individual PRM graph, to guarantee near-optimality using the tensor roadmap. 
This significantly improves upon a previous asymptotic analysis, wherein the number of samples tends to infinity. Our new finite sample-size analysis supports guaranteed high-quality solutions in practice within finite time.
To achieve our new result, we first develop a sampling scheme, which we call the \emph{staggered grid}, for finite-sample motion planning for individual robots, which requires significantly fewer samples than previous work.  We then extend it to the much more involved MRMP setting which requires to account for interactions among multiple robots. Finally, we report on a few experiments that serve as a  verification of our theoretical findings and raise interesting questions for further investigation.  
\end{abstract}


\maketitle

\else 

\iftrov
\begin{abstract}
An underlying structure in several sampling-based methods for continuous multi-robot motion planning (MRMP) is the \emph{tensor roadmap}, which emerges from combining multiple probabilistic roadmap (PRM) graphs constructed for the individual robots via a tensor product. We study the conditions under which the tensor roadmap encodes a near-optimal solution for MRMP---satisfying these conditions implies near optimality for a variety of popular planners, including dRRT*, and the discrete methods M* and conflict-based search, when applied to the continuous domain. 
We develop the first finite-sample analysis of this kind, which specifies the number of samples, their deterministic distribution, and magnitude of the connection radii that should be used by each individual PRM graph, to guarantee near-optimality using the tensor roadmap. 
This significantly improves upon a previous asymptotic analysis, wherein the number of samples tends to infinity. Our new finite sample-size analysis supports guaranteed high-quality solutions in practice within finite time.
To achieve our new result, we first develop a sampling scheme, which we call the \emph{staggered grid}, for finite-sample motion planning for individual robots, which requires significantly fewer samples than previous work.  We then extend it to the much more involved MRMP setting which requires to account for interactions among multiple robots. Finally, we report on a few experiments that serve as a  verification of our theoretical findings and raise interesting questions for further investigation.
\end{abstract}

\else 

\begin{abstract}
An underlying structure in several sampling-based methods for continuous multi-robot motion planning (MRMP) is the tensor roadmap (TR), which emerges from combining multiple PRM graphs constructed for the individual robots via a tensor product. We study the conditions under which the TR encodes a near-optimal solution for MRMP---satisfying these conditions implies near optimality for a variety of popular planners, including dRRT*, and the discrete methods M* and CBS when applied to the continuous domain. 
We develop the first finite-sample analysis of this kind, which specifies the number of samples, their deterministic distribution, and magnitude of the connection radii that should be used by each individual PRM graph, to guarantee near-optimality using the TR. 
This significantly improves upon a previous asymptotic analysis, wherein the number of samples tends to infinity. Our new finite sample-size analysis supports guaranteed high-quality solutions in practice within finite time.
To achieve our new result, we first develop a sampling scheme, which we call the \emph{staggered grid}, for finite-sample motion planning for individual robots, which requires significantly less samples than previous work.  We then extend it to the much more involved MRMP setting which requires to account for interactions among multiple robots. Finally, we report on a few experiments that serve as a  verification of our theoretical findings and raise interesting questions for further investigation.  
\end{abstract}

\fi 

\fi 

\section{Introduction}
Multi-robot (MR) systems are already playing a crucial role in manufacturing, warehouse automation, and natural resource monitoring, and in the future they will be employed in even broader domains from space exploration to search-and-rescue. One of the most basic ingredients necessary in all those applications are mechanisms for multi-robot motion planning (MRMP), which should quickly generate motion trajectories to move robots from their origins to destinations, while avoiding collisions with the environment and between robots. In many cases, it is desirable to develop MRMP approaches that provide strong guarantees of completeness and near-optimality, to ensure that a high-quality solution would be found (if one exists). To achieve this, methods for MRMP must accurately capture the continuous state space of individual robots and the intricate interactions between multiple robots. Those considerations make the task of designing efficient high-quality methods for MRMP tremendously challenging~\indirectcite{SolHal16j,sy-snp84,hss-cmpmio,johnson2018relationship}

In this work we develop a general approach for centralized near-optimal sampling-based motion planning using a finite number of samples, which is the first of its kind. 
Previous solutions guarantee path quality for the multi-robot case only asymptotically, namely, as the number of samples tends to infinity. In contrast, our new method explicitly prescribes finite samples sets, and by that supports guaranteed high-quality solutions in practice, within bounded running time.

\fakeparagraph{Related work.}
This paper focuses on centralized approaches for  multi-robot motion coordination where the planning is carried out by a single entity that knows the state of the entire system. In the decentralized setting, which is outside the scope of this work, each robot has autonomy in planning its motion typically based on local knowledge about the system (say the state of its few neighboring robots)~\indirectcite{long2018towards, yan2013survey, 5663672, DBLP:journals/corr/abs-2103-11067, DBLP:journals/trob/RufliAS13}. A common approach to centralized MRMP, which is often taken in the AI research community, is to consider a discretized version of the problem, termed multi-agent pathfinding (MAPF), wherein robots are assumed to move along vertices of a graph. A variety of methods were developed for MAPF, including integer-programming formulations~\indirectcite{han2019effective}, path-based search methods~\indirectcite{WagCho15}, and conflict-based search~\indirectcite{sharon2015conflict}. Those have been successfully applied to a variety of problems involving multiple robots---from warehouse management~\indirectcite{MaEA2019} to multi-drone package delivery~\indirectcite{ChoudhuryEA2020}. Unfortunately, MAPF methods provide no solution quality guarantees with respect to the original continuous MRMP problem, since they usually employ a crude discretization of the robots' environment in the form of a regular lattice, where a cell size corresponds to a robot's bounding box. 

Recent work extends conflict-based search to the continuous MRMP domain, by using \prm graphs to capture 
the individual robots' state space~\indirectcite{HonigEA2017, LiEA2019,IrivingEA2019}. Nevertheless, no guidelines are provided as to how to construct those roadmaps (in terms of number  of samples, their distribution, and connection radius) to guarantee completeness or optimality  with respect to the original MRMP problem.

In a different line of work, computational geometry methods are employed to explicitly reason about the robots' continuous state space. Such methods are quite powerful, in that they guarantee polynomial runtime, completeness, and near-optimal solutions for different quality metrics~\indirectcite{AdlETAL15,SolomonHalperin2018,TurMicKum12,SolETAL15}. However, those methods are typically restricted to disc-shaped robots operating within a planar domain, and they require special separation constraints, e.g., between the robots' initial and terminal positions, in order to work correctly, which limits their applicability in practice.

A promising direction, which aims to overcome the limitations of  discrete and geometry-based approaches, are sampling-based (SB) planners. SB-planners were initially developed to the tackle the single-robot motion-planning problem for complex systems, by using random sampling of states to capture the structure of the robot's complex state space, which results in a discrete graph representation. Many of those algorithms, including the celebrated \prm~\indirectcite{PRM.508439,660866,lavalle2006planning}, and \rrt~\indirectcite{KL2000_geom,KSLBH18} algorithms are known to converge to a solution asymptotically with the number of samples drawn. Some SB-planners are also guaranteed to converge to the optimal solution as the number of samples tends to infinity~\indirectcite{doi:10.1177/0278364915577958,karaman2011sampling,SoloveyEA2020,HauserZ16,KleinbortEA202}. Although SB planners typically utilize randomized sampling approaches, some papers have considered more sophisticated methods such as deterministic and quasi-random sampling~\cite{DBLP:journals/ral/PalmieriBMA20, 932820,DBLP:journals/ijrr/JansonIP18} and learning-based sampling~\cite{DBLP:conf/icra/IchterHP18}. Our previous work~\cite{TsaoSoloveyETAL2020} studies the sample complexity of \prm and derives theoretical upper and lower bounds. Notably, those are the first bounds of their kind derived for the finite-sample setting, in contrast with previous works that consider the asymptotic regime where the number of samples tends to infinity. For additional discussion on sampling distributions for the single-robot case see~\cite{TsaoSoloveyETAL2020}.

Following the success of SB approaches for the single-robot case a variety of SB-MRMP approaches have emerged, from methods that aim to apply \prm-based solutions~\indirectcite{sl-upp,SveOve98}, to techniques that sample local instances of the MAPF problem~\indirectcite{sh-kcolor14,KronitirisEA2014}. A different method, termed discrete RRT (\drrt)~\indirectcite{drrt.doi:10.1177/0278364915615688}, carefully explores an implicitly-represented tensor roadmap (TR), which emerges from combining several \prm roadmaps constructed for the individual robots, to effectively solve instances of MRMP requiring tight coordination between multiple robots. A recent work further improves this approach with the \drrtstar method~\indirectcite{DBLP:journals/arobots/ShomeSDHB20}, which is also shown to yield a near-optimal MRMP solution. To the best of our knowledge, this is the only scalable SB-MR planner with such a guarantee. Unfortunately, this result is asymptotic and does not specify guarantees for a finite number of samples. 

\fakeparagraph{Contribution. }
We develop a general framework for near-optimal SB-MRMP while using a finite number of samples. To this end, we study the structure of the TR which is an underlying ingredient in \drrtstar, \mccbs~\indirectcite{LiEA2019} (a continuous-space extension of \cbs), and \mstar~\indirectcite{WagCho15} (when applied to a continuous space). 
We develop conditions under which the TR encodes a near-optimal solution to MRMP---if those conditions hold, it implies that the aforementioned planners are guaranteed to be near-optimal as well. In particular, we prescribe a recipe for constructing individual-robot \prm \footnote{Throughout the paper we slightly abuse the term \prm to refer generally to maps of individual robots, even when they are deterministic.} graphs, in terms of number of samples, their deterministic distribution, and connection radius, so that the resulting TR encodes a near-optimal solution for MRMP. 

To achieve this we refine the asymptotic analysis that we developed in~\directcite{DBLP:journals/arobots/ShomeSDHB20} (in the context of the \drrtstar algorithm) to the finite-sample regime. The latter requires a much more careful study of the different ingredients of the problem, including the clearance parameter between robots and obstacles, and the approximation factor, which were previously assumed to be infinitesimally small. A key component in our result is a new  sampling scheme that we develop, which we call the \emph{staggered grid}, for finite-sample motion planning for individual robots. This sampling scheme requires significantly less samples than our previous work~\indirectcite{TsaoSoloveyETAL2020}, which employs a greedy approach for generating samples, to achieve near-optimality for the single-robot case. We also emphasize that the previous work~\indirectcite{TsaoSoloveyETAL2020} has not considered the multi-robot setting. For more details on~\indirectcite{TsaoSoloveyETAL2020}, see Section~\ref{sec:single_comparison}.
Throughout this work we assume that all the robotic systems are holonomic, where the configuration space of each robot is $[0,1]^d$, for some $d\geq 2$.

The organization of this paper is as follows. In Section~\ref{sec:single} we describe the single-robot problem,  introduce the staggered grid, and study the theoretical properties of \prm using this sampling scheme. In Section~\ref{sec:multi_main} we describe our central contribution, namely the extension of this theoretical result to the multi-robot setting. We provide experimental results in Section~\ref{sec:experiments}, and conclude with an outline of future work in Section~\ref{sec:conclusion}.

\ifextendedv
This paper is an extended and revised version of a conference version that previously appeared in the International Conference on Robotics and Automation~\cite{DayanSoloveyETAL2021}. 
The changes from the conference version include additional full proofs for all theorems and lemmas (Section~\ref{sec:proof_single} and most of Section~\ref{sec:multi completeness}). Importantly, most of the major proofs have only appeared previously as sketches of proofs. We also add the entirety of Section~\ref{sec:single_comparison}, which compares our results using the staggered grid with previous results in terms of the sample-set size, and Section~\ref{sec:random_vs_sg}, which compares the staggered grid with random sampling. 
\fi

\section{Improved sampling distributions for a single robot}\label{sec:single}

We present improved sampling distributions for probabilistic roadmaps (\prm) for the \emph{single-robot} case. Our results prescribe the number of samples and their distribution sufficient to achieve a desired solution quality using the \prm approach. First, we introduce basic ingredients of the problem, then review the concept of \prm, and introduce our staggered-grid sampling. We conclude  with our main theoretical result for the single-robot case, pertaining to the quality of solutions obtained using \prm with staggered-grid sampling.

\subsection{Basics of single-robot motion planning}
Let $\C$ denote the configuration space of the robot, which we assume here to be $[0,1]^d$, where $d$ is the dimension, which is the number of degrees of freedom of the robot. The free space, denoted by $\C^f \subset \C$, represents the set of all collision-free configurations. A motion planning problem is a tuple $\M := (\C^f, x^{s},  x^{g})$, where $\C^f$ is the free space, and $x^{s},  x^{g}\in \C^f$ are the  start and goal configurations, respectively. A solution of $\M$ is a continuous collision-free trajectory $\sigma:[0,1]\to\C^f$ that begins at $\sigma(0) = x^{s}$ and ends at $\sigma(1) = x^{g}$. We measure the quality of a trajectory $\sigma$ by its length, which is denoted by $\norm{\sigma}$.

A crucial property of trajectories in sampling-based planning is the notion of clearance. A trajectory $\sigma$ has $\delta$-clearance if $\bigcup_{0 \leq t \leq 1}B_{\delta}(\sigma(t)) \subseteq \C^f$, for $B_{\delta}(\sigma(t))$ being the $d$-dimensional closed Euclidean ball with radius $\delta$ centered at $\sigma(t)$. We say that  $\M$ is $\delta$-clear if there exists a trajectory $\sigma$  with clearance $\delta$ that solves $\M$.

\subsection{Probabilistic roadmaps and sample sets}
\label{ssec:roadmaps-and-samples}
We provide a formal definition of the Probabilistic Roadmap (\prm) method~\indirectcite{PRM.508439}, which constructs a discrete graph that captures the connectivity of $\C^f$ via sampling. \prm plays a critical role in various sampling-based planners (see, e.g.,~\indirectcite{solovey2020critical,TsaoSoloveyETAL2020}). \prm is also instrumental to our result both on single-robot motion-planning in this section and on multi-robot motion-planning in Section~\ref{sec:multi_main}.

For a given motion-planning problem $\M = (\C^f, x^{s},  x^{g})$, a sample (point) set $\X\subset \C^f$, and a connection radius $r>0$, \prm generates a graph denoted by $G_{\M(\X, r)} = (V,E)$. The vertex set $V$ consists of all the collision-free configurations in $\X \cup \{ x^{s},  x^{g}\}$. The set of (undirected) edges, $E$, consists of all vertex pairs $v, u \in V$ such that the Euclidean distance between them is at most $r$, and the straight-line segment between them is collision-free. 
Formally, we define 
\begin{align*}
V:=  &(\X\cup \{ x^{s},  x^{g}\}) \cap \C^f {\rm , and} \\
E:=& \left\{ \{v, u\} \in V\times V: \norm{v-u}\leq r, \CH(\{v, u\}) \subset  \C^f \right\},
\end{align*}
where $\CH$ denotes the convex hull of a point set. 

To measure the quality of a sample set $\X$ and a connection radius $r$ we use the following definition~\indirectcite{TsaoSoloveyETAL2020}.
\theoremstyle{definition}
\begin{definition}[Single-robot ($\eps,\delta$)-completeness] Given a sample set $\X$ and connection radius $r$, we say that $(\X, r)$ is ($\eps,\delta$)-complete for some stretch $\eps>0$ and clearance parameter $\delta>0$ if for every $\delta$-clear $\M= (\C^f, x^{s},  x^{g})$ it holds that
$$d({G_{\M(\X, r)}},x^{s},  x^{g}))\leq (1+\eps)\OPT_\delta,$$ where $d({G_{\M(\X, r)}},x^{s},  x^{g})$ denotes the length of the shortest trajectory from $x^{s}$ to $x^{g}$ in the graph $G_{\M(\X, r)}$, and $\OPT_\delta$ is the length of the shortest $\delta$-clear solution to $\M$.
\end{definition}

Notice that the solution induced by ${(G_{\M(\X, r)}},x^{s},  x^{g})$ must be collision free, albeit its clearance can be smaller than $\delta$. This also applies to the generalized definition of $(\eps,\delta)$-completeness in the multi-robot setting (Section~\ref{sec:multi_main}).

\subsection{Efficient sampling via a staggered grid}
In preparation for our main result in this section we introduce a new sampling scheme termed a \emph{staggered grid}, which we denote by $\X_{\beta,\gamma}$,  and study its implications with respect to ($\eps$, $\delta$)-completeness of \prm. Refer to Figure~\ref{fig:cover_lemma_fig} for an illustration. An important property of this staggered grid, which we prove in \ifextendedv
Lemma~\ref{lem:beta_cover} below, 
\fi
is that by placing $\beta$-radius hyperspheres  centered at the points $\X_{\beta,\gamma}$ we obtain a coverage of the robot's configuration space. Moreover, the size of the set $\X_{\beta,\gamma}$ is smaller than previously obtained sets with similar coverage properties~\indirectcite{TsaoSoloveyETAL2020}.
Notice that we use a slightly shrunk instance of the unit hyper-cube; a solution path should not pass too close to the boundary of the hyper-cube to respect the clearance condition---this is the role of the parameter $\gamma$.

\begin{definition}[Staggered grid]\label{def:staggered_grid}
For given $\beta>0,\gamma>0$ the staggered grid $\X_{\beta,\gamma}$ is the union of two point sets $\X^1_{\beta,\gamma}, \X^2_{\beta,\gamma}$ in $[0,1]^d$, where
\ifextendedv
\begin{align*}
\X^1_{\beta,\gamma}  = \bigg\{&(p_1, p_2,...,p_d) :
p_i=\gamma+(2k-1)w,\\&\; 1\leq k\leq \left\lceil\frac{1-2\gamma}{2w}\right\rceil, \; 1 \leq i \leq d\bigg\},\\
\X^2_{\beta,\gamma} = \bigg\{&(p_1, p_2,...,p_d) : p_i=\gamma+2kw,\\&\; 0\leq k\leq \left\lceil \frac{1-2\gamma}{2w}\right\rceil, \; 1 \leq i \leq d\bigg\},
\end{align*}
and $w=\frac{\beta\sqrt{2}}{\sqrt{d}}$.
\fi
\end{definition}

Notice that both
$\X^1_{\beta,\gamma}$  and $\X^2_{\beta,\gamma}$ are square grids of side length $2w$, and the total number of points in the construction is 
$\left(\left\lceil\frac{(1-2\gamma)\sqrt{d}}{\sqrt{8}\beta }\right\rceil\right)^d+\left(\left\lceil\frac{(1-2\gamma)\sqrt{d}}{\sqrt{8}\beta }\right\rceil+1\right)^d$.
\ifextendedv
We illustrate the staggered grid in Fig~\ref{fig:cover_lemma_fig}.  

We mention that for two-dimensional problems ($d=2$) the staggered grid can be viewed as a standard square grid that is rotated by $\pi/2$ (after which vertices outside $[0,1]^2$ are removed). However, in higher dimensions $d\geq 3$ this is not case, i.e., the staggered grid cannot be constructed by simply rotating a square grid. To see this, consider for example the setting of $d=3$. The shortest non-zero vector in the standard square grid is realized between two grid nodes along the same axis (e.g., unit length in the integer grid). 
In 3D, there are six shortest vectors from the origin in the standard grid (in general, $2d$ shortest vectors in $\dR^d$).
In the staggered grid in 3D (assume the first grid has the origin as a node), there are nearest nodes in each of the points $(\pm 0.5,\pm 0.5, \pm 0.5)$, namely eight shortest vectors.

    \ifextendstaggered
    \dror{New text here.}
    A detailed size analysis for the staggered grid can be found in Section~\ref{sec:staggered_grid_analysis}.
    \fi
\fi

\ifextendedv
\newcommand{\coverLemmaFigWidth}{0.49\linewidth}
\fi

\begin{figure}[b]
\vspace{5pt}
  \centering
  \begin{subfigure}[b]{\coverLemmaFigWidth}
  \centering
    \includegraphics[trim=100 53 70 25, clip, width=1\textwidth]{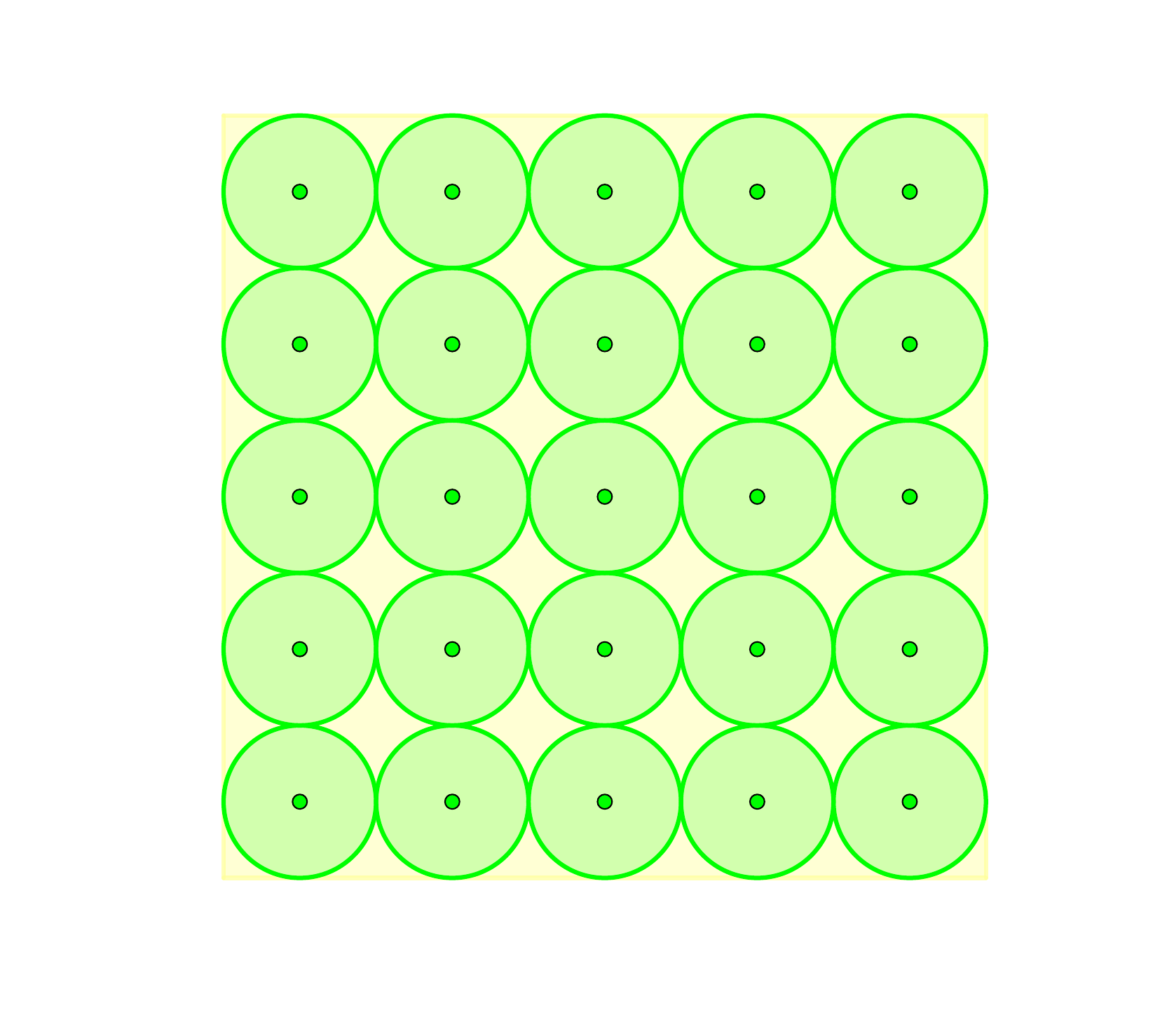}
  \end{subfigure}
  \begin{subfigure}[b]{\coverLemmaFigWidth}
  \centering
    \includegraphics[trim=100 53 70 25, clip, width=1\textwidth]{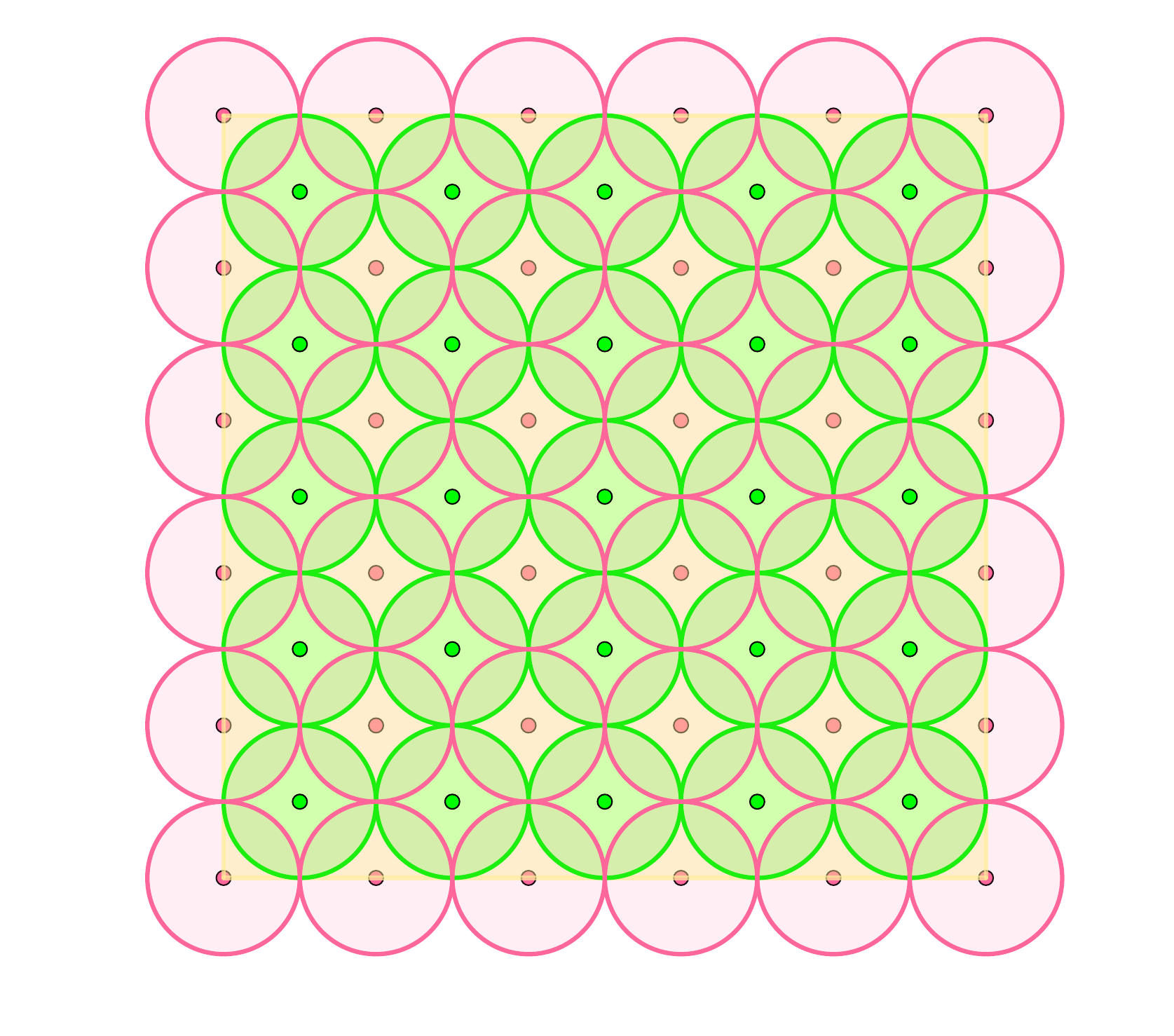}
  \end{subfigure}

  \caption{\sf{Illustration of Definition~\ref{def:staggered_grid} in two dimensions, for $\beta=0.08$, $\gamma=0.1$. On the left  we visualize the first layer $\X^1_{\beta,\gamma}$ of the staggered grid (green discs). On the right we add the second layer $\X^2_{\beta,\gamma}$ (red discs). The centers of the discs are the points of $\X_{\beta,\delta}$.}}
  \label{fig:cover_lemma_fig}
\end{figure}

Next we make the connection between the staggered grid and a \prm graph that is ($\eps$, $\delta$)-complete. Namely, we will show that given $\eps$ and $\delta$, there is a sample set and radius ($\X_{\beta,\gamma}$,$r$) that are  ($\eps$, $\delta$)-complete, where each of $\beta$ and $r$ depends on both $\eps$ and $\delta$, and $\gamma$ is equal to $\delta$. 

\begin{theorem}[Sufficient conditions for ($\eps,\delta$)-completeness]\label{thm:main_single}
Fix a stretch parameter $\eps>0$ and clearance $\delta > 0$.  For a sampling distribution $\X=\X_{\alpha \delta,\delta}$, where  $\alpha=\frac{\eps}{\sqrt{1+\eps^2}}$, and the radius $r=\frac{2(\eps+1)}{\sqrt{1+\eps^2}}\delta$,
it follows  that $(\X,r)$ is ($\eps,\delta$)-complete.
\end{theorem}
\ifextendedv
\subsection{Proof of Theorem~\ref{thm:main_single}}\label{sec:proof_single}
In this section we provide a proof of Theorem~\ref{thm:main_single}. First, we introduce the concept of $\beta$-cover \indirectcite{herbrich2001learning}. Informally, a $\beta$-cover of a domain is a set $\X$ of points such that no point of the domain is too far from some point in $\X$ (defined formally below). Next we show that the staggered grid $\X_{\beta,\delta}$ is a $\beta$-cover of $[\delta,1-\delta]^d$. Then we exploit this property in the proof for Theorem~\ref{thm:main_single}. 

\begin{definition}
For a given $\beta>0$, a set $\X\subset \dR^d$ is a $\beta$-cover\footnote{The common naming for this term is \xcover{$\eps$}. We use the letter $\beta$ to avoid confusion with the stretch factor which is called $\eps$ in this paper. We also note that the term \xcover{$\beta$} is a more accurate substitution for the term ``$\beta$-net'', which we used in our previous work~\indirectcite{TsaoSoloveyETAL2020}.} for a set $A \subset \dR^d$ if for every $a \in A$, there exists $s \in \X$ such that $\|a-s\|\leq \beta$
\end{definition}

First, we prove that $\X_{\beta,\gamma}$ forms a $\beta$-cover over $[\gamma,1-\gamma]^d$.

\begin{lemma}\label{lem:beta_cover}
Let $\beta>0,\gamma>0$ and define $A_\gamma=[\gamma,1-\gamma]^d$.  Then the staggered grid $\X_{\beta,\gamma}$ is a $\beta$-cover for $A_\gamma$.
\end{lemma}

\begin{proof}
Recall that the width $w$ of a cell in each of the two grids that form the staggered grid $\X_{\beta,\gamma}$ is  $\beta\sqrt{2}/\sqrt{d}$. We divide the $d$-dimensional $(1-2\gamma)$-hypercube, $A_\gamma$, by hyperplanes as follows.
For each coordinate $x_i$ we define the hyperplanes $x_i = \gamma + 2kw$, for $k\in \{0,1,\ldots,\left\lceil\frac{1-2\gamma}{2w}\right\rceil\}$. Jointly, those hyperplanes induce a partition of $A_\gamma$ into a set of hypercubes $H$: each hypercube $h\in H$ is of edge length $2w$, its center is a point from $\X^1_{\beta,\gamma}$, and its vertices are points from $\X^2_{\beta,\gamma}$.
See Figure~\ref{fig:beta_cover1} for an illustration in $\dR^2$.
\begin{figure}[H]
  \centering
  \begin{subfigure}[b]{0.49\linewidth}
  \centering
    \includegraphics[trim=180 80 360 30, clip, width=1\textwidth]{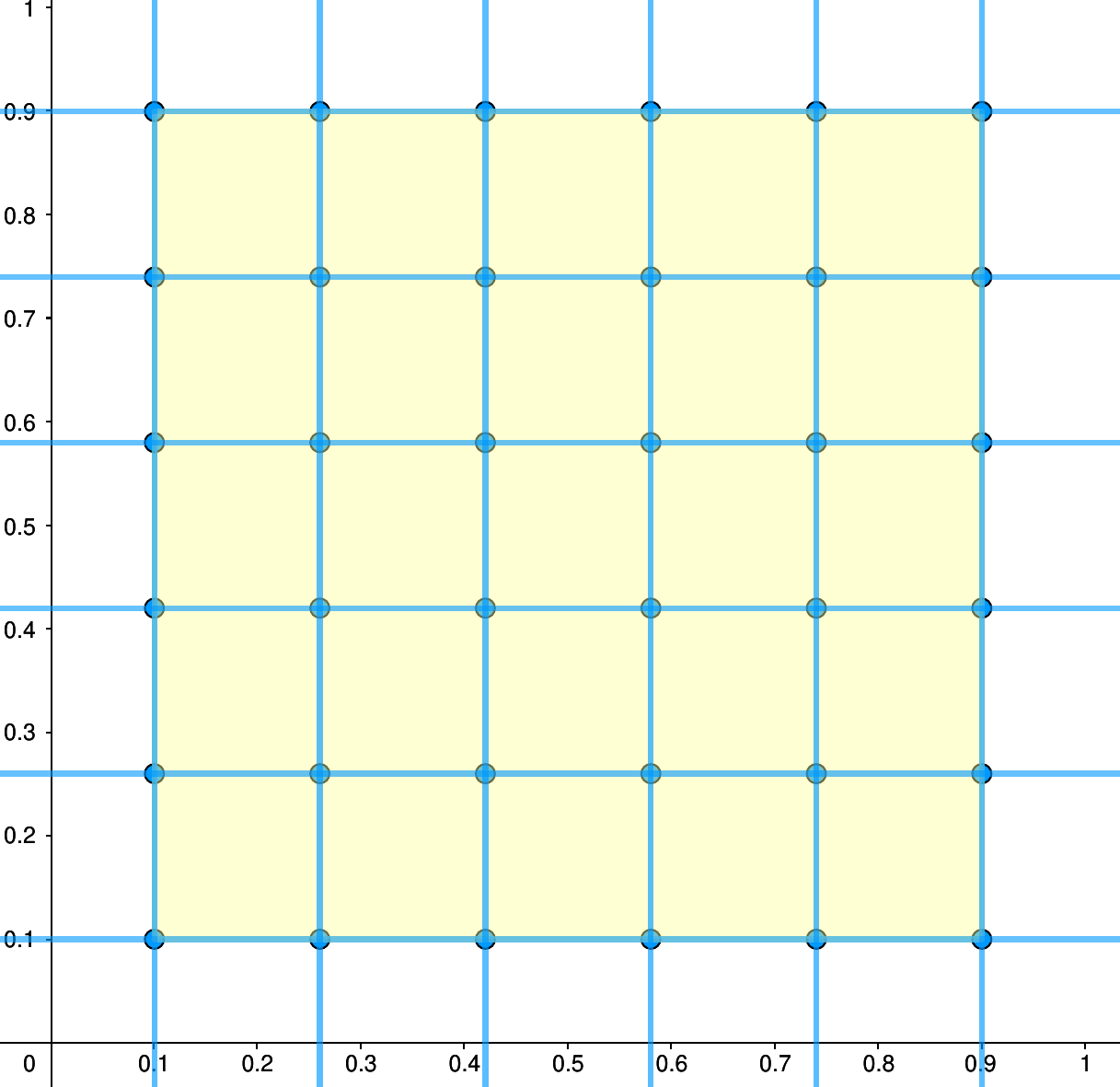}
  \end{subfigure}
  \begin{subfigure}[b]{0.49\linewidth}
  \centering
    \includegraphics[trim=180 80 360 30, clip, width=1\textwidth]{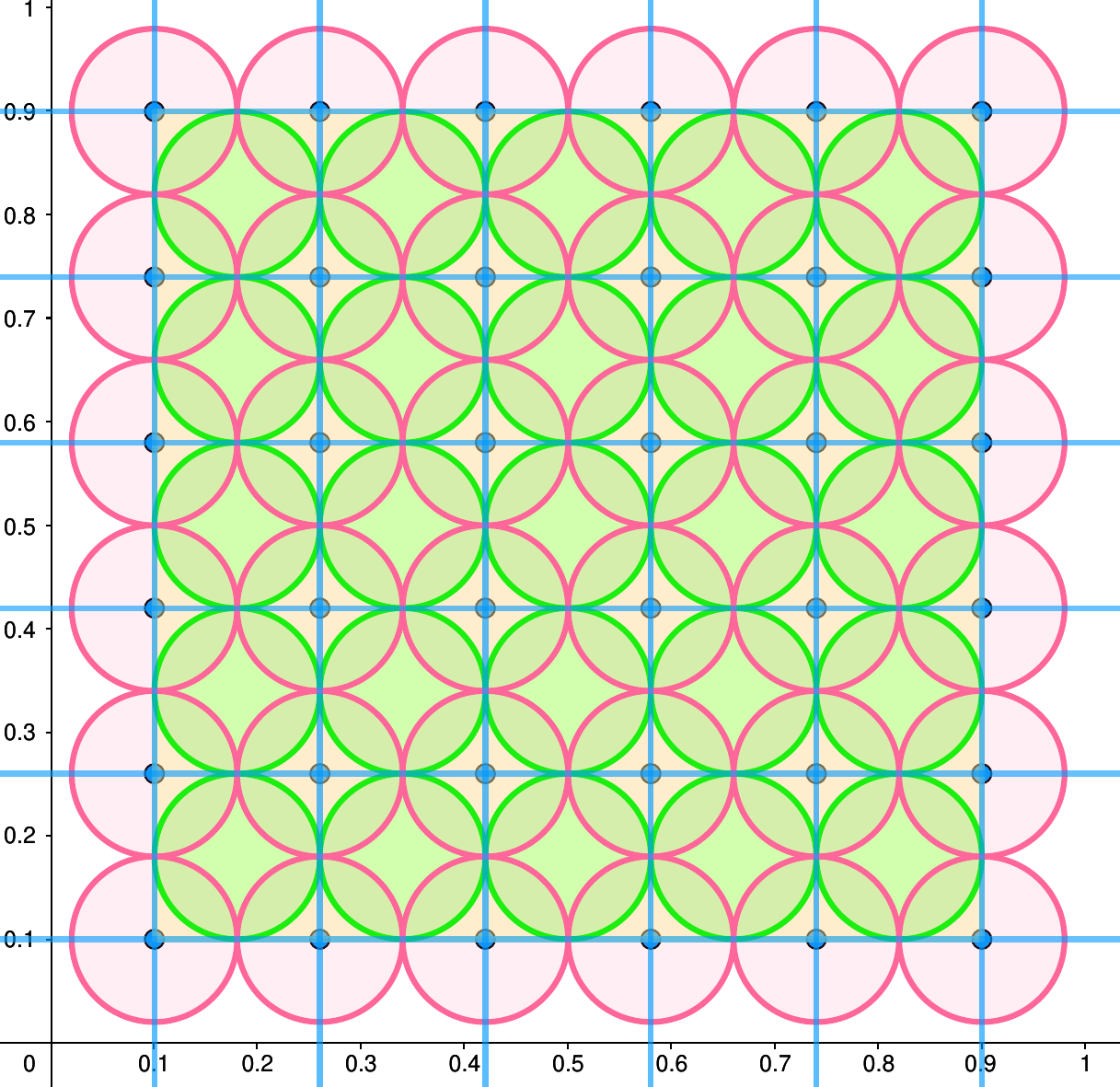}
  \end{subfigure}

  \caption{\sf{Illustration of the hyperplanes partition for Lemma~\ref{lem:beta_cover} in two dimensions, for $\gamma=0.1$, $\beta=0.08$, and $w=0.08$. In both sub-figures the yellow cube represents $[\gamma,1-\gamma]^d$, and the blue lines are the hyperplanes orthogonal to $x_i$. On the left we demonstrate how the cube is split into smaller $2w$-cubes. On the right we demonstrate that these $2w$-cubes are covered with $\beta$-balls. The green balls are $\beta$-balls centered at points of  $\X^1_{\beta,\gamma}$, while the red balls are $\beta$-balls centered at points of $\X^2_{\beta,\gamma}$.}}
  \label{fig:beta_cover1}
\end{figure}

We now prove that each $h\in H$ induced by our hyperplanes is fully covered using the $\beta$-balls centered at its vertices and at its center.
Formally, for a given $h\in H$, let $S_h$ be the set of points comprising 
the vertices of $h$ and the center point of $h$ (i.e., $\abs{S_h}=2^d+1$). We show that $S_h$ is a \xcover{$\beta$} for $h$. 

Let $s\in S_h$ be the center point of $h$. Without loss of generality, assume that $s$ is at the origin. Let $B_{\beta}(s)$ be the $\beta$-ball centered at $s$. Let $p=(p_1,..,p_d)$ be a $d$-dimensional point such that $p\in h$ and $p \notin B_{\beta}(s)$. We further assume that $\forall i\in[1..d], 0\leq p_i \leq w$; all the other cases can be handled symmetrically. See Figure~\ref{fig:beta_cover2} for an illustration in $\dR^2$. Thus, since $p \notin B_{\beta}(s)$ 
we have that 
$\norm{s - p}>\beta$, which implies that  $\beta^2 < \sum_{i=1}^d p_i^2$.

Next, recall that $s_w:=(w,... ,w)\in S$ since the hypercubes edge length is $2w$. We upper bound the distance of $p$ from $s_w$ in the following manner: 
\begin{align*}
\norm{s_w - p}^2 &=\sum_{i=1}^d (w-p_i)^2=\sum_{i=1}^d  (w^2-2wp_i +p_i^2) \\ & \leq \sum_{i=1}^d  (w^2-2p_i^2 +p_i^2)\\
&= \sum_{i=1}^d w^2 -\sum_{i=1}^d p_i^2<\sum_{i=1}^d w^2 - \beta^2,
\end{align*}
where the first inequality follows from the fact that $p_i \leq w$, and the second inequality follows from $\beta^2 < \sum_{i=1}^d p_i^2$.

Finally, by subsisting $w=\beta\sqrt{2}/\sqrt{d}$, we obtain
\[\norm{s_w - p}^2 < \sum_{i=1}^d\frac{2\beta^2}{d}-\beta^2 = \beta^2,\] 
which implies that $p$ is covered by the $\beta$-ball centered at $s_w$.
\begin{figure}[ht]
  \centering
   \includegraphics[width=0.6\linewidth]{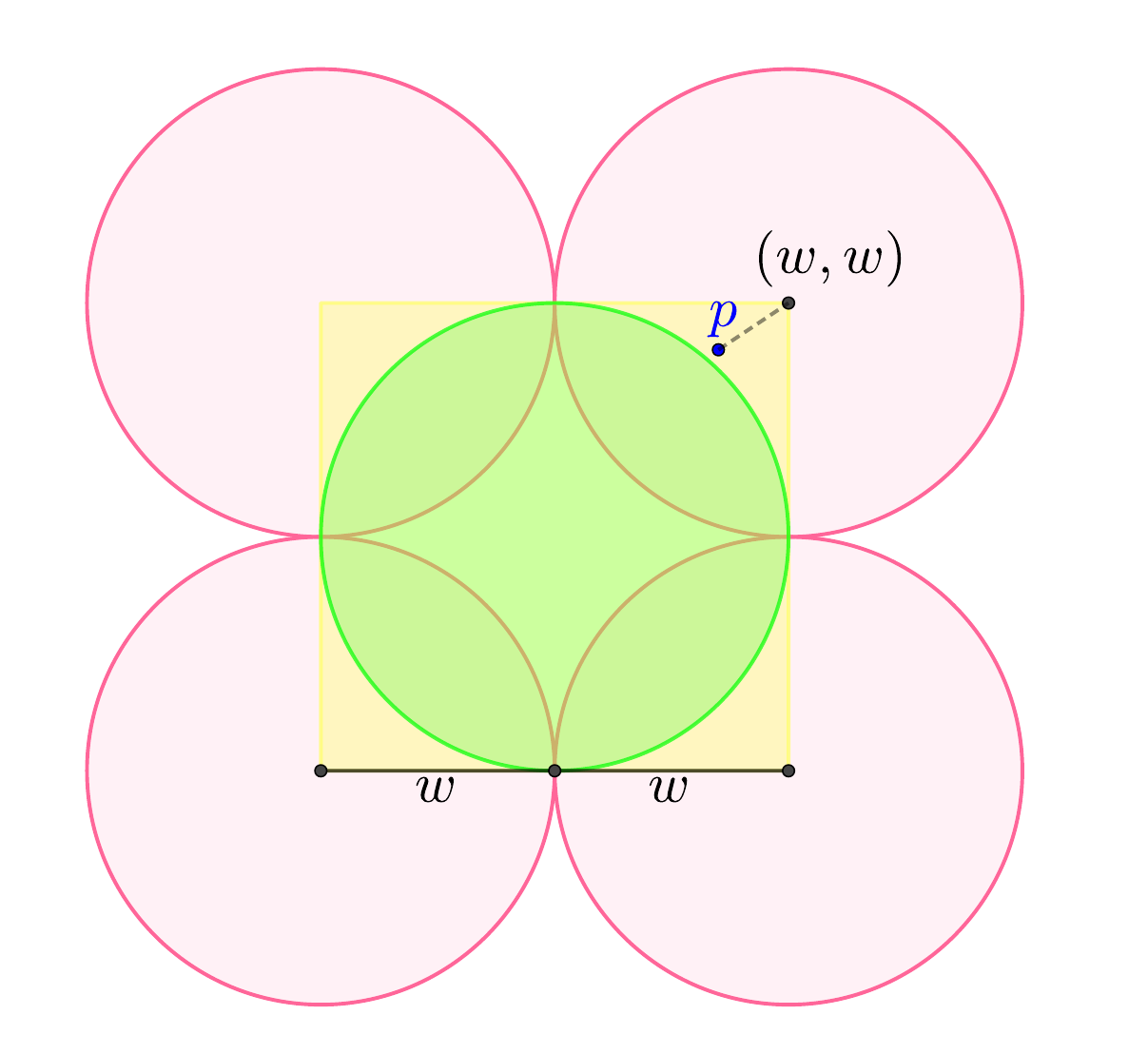}

  \caption{\sf{Illustration of the covering of a $2w$-hypercube of Lemma~\ref{lem:beta_cover} in two dimensions. The point $p$ is the point tested in the lemma. The red balls are $\beta$-balls centered on the cube vertices. The green ball is the $\beta$-ball centered in the cubes center called $B_{\beta}(s)$. The dashed line represents the distance between $p$ and $(w, w)$. We note that in higher dimensions the red balls would overlap as $w$ becomes smaller in relation to $\beta$.}}
  \label{fig:beta_cover2}
\end{figure}

We have proven that $\X_{\beta,\gamma}$ provides a \xcover{$\beta$} for each $h\in H$, which also guarantees that $\X_{\beta,\gamma}$ is a \xcover{$\beta$} for the larger hypercube $[\gamma,1-\gamma]^d$. Thus we conclude that $\X_{\beta,\gamma}$ is a \xcover{$\beta$} for $A_\gamma$.
\ifijrrv
\qed
\fi
\end{proof}

Next, we exploit the above property for showing that by appropriately setting the values of $\beta, \gamma$ and $r$ with respect to $\eps,\delta$, we obtain a good approximation of any collision-free trajectory, which implies that $(\X_{\beta,\gamma},r)$ is $(\eps,\delta)$-complete. To do so, we prove the following lemma, which is an extended version of~\indirectcite[Theorem~2]{TsaoSoloveyETAL2020}. This lemma, will also be instrumental in extending the theory we develop for the single-robot case, to the multi-robot setting. In order to state the theorem, we introduce  additional  notation regarding a sequence of points along a given trajectory of certain spacing. 

\begin{definition}
\label{def:trajectory-timing}
Let $\sigma$ be a collision-free trajectory, and let $\rho>0$ be a step size such that $\|\sigma(0)-\sigma(1)\|\geq \rho$. Then the time sequence $T_\sigma^\rho=(\tau_0,\tau_1,\ldots,\tau_\ell)$, where $\tau_0<\tau_1\ldots<\tau_\ell$, is defined in the following manner for some $\ell\in \dN_{>0}$: $\tau_0=0,\tau_\ell=1$, for all $1\leq i \leq \ell-1$, $\|\sigma(\tau_i)-\sigma(\tau_{i-1})\|=\rho$, and $\|\sigma(\tau_\ell)-\sigma(\tau_{\ell-1})\|\leq \rho$.
\end{definition}

\begin{lemma}\label{lem:super_duper}
For given $\beta,\gamma,\rho>0$, such that $\beta^2+(\rho/2)^2\leq \gamma^2$, let $\sigma$ be a $\gamma$-clear solution for $\M= (\C^f, x^{s},  x^{g})$, and let $T_\sigma^\rho=(\tau_0,\tau_1,\ldots,\tau_\ell)$  be a time sequence as in Definition~\ref{def:trajectory-timing}. Then there exists a point set $Z_\sigma^\rho=(z_0,\ldots,z_\ell)$, where $z_i\in \X_{\beta,\gamma}\cup \{x^{s},  x^{g}\}$ for all $0\leq i\leq \ell$, such that the following properties holds:
\begin{enumerate}[label=(\roman*)]
    \item $\|z_i-\sigma(\tau_i)\|\leq \beta$, for all $0\leq i\leq \ell$;
    \item $\|z-\sigma(\tau_{i-1})\|\leq \beta +\rho$, for all $1\leq i\leq \ell$, $z \in \CH\left(z_{i-1},z_i\right)$;
    \item $\|z-\sigma(\tau_i)\|\leq \beta +\rho$, for all $1\leq i\leq \ell$, $z \in \CH\left(z_{i-1},z_i\right)$;
    \item $\|z_i-z_{i-1}\|\leq 2\beta +\rho$, for all $1\leq i\leq \ell$;
    \item $\CH\left(z_{i-1},z_i\right)$ is collision free, for all $1\leq i\leq \ell$. That is $\CH\left(z_{i-1},z_i\right) \subset \C^f$, for all $1\leq i\leq \ell$;
    \item $\sum_{1\leq i\leq \ell}\|z_i-z_{i-1}\|\leq \left(1+\frac{2\beta}{\rho}\right)\norm{\sigma}$.
\end{enumerate}
\end{lemma}

\begin{proof}
Define, as before,  $A_\gamma=[\gamma,1-\gamma]^d$. Notice that  since $\sigma$ is $\gamma$-clear then $\sigma(t) \in A_\gamma$ for all $0\leq t \leq 1$ (and in particular for all $\tau_i\in T_\sigma^\rho$). 
Next, we address each individual component of the theorem. \vspace{5pt}

\noindent \emph{Property~(i)}: This follows directly from Lemma~\ref{lem:beta_cover} as $\X_{\beta,\gamma}$ is a \xcover{$\beta$} for $A_\gamma$ and $\sigma(\tau_i) \in A_\gamma$.\vspace{5pt}

\noindent \emph{Properties~(ii) and (iii)}: First, we upper bound the expression $\|z_i-\sigma(\tau_{i-1})\|$ (the proof for Property~(iii) is symmetric). By definition of  $T_\sigma^\rho$  we have that  $\|\sigma(\tau_i)-\sigma(\tau_{i-1})\|=\rho$, for $1\leq i \leq \ell-1$. Thus, using Property~(i), we have $\|z_i-\sigma(\tau_i)\|\leq \beta$. Using the triangle inequality it follows that, 
\begin{align*}
\|z_i-\sigma(\tau_{i-1})\| &\leq \|z_i-\sigma(\tau_i)\| + \|\sigma(\tau_i)-\sigma(\tau_{i-1})\| \\ & \leq \beta + \rho.
\end{align*}

Let $z \in \CH\left(z_{i-1},z_i\right)$, that is $z=\lambda z_{i-1}+(1-\lambda)z_i$ for some $\lambda$ such that $0\leq \lambda \leq 1$. Thus,
\begin{align*}
    &\norm{z-\sigma(\tau_{i-1})} \\ &= \norm{\lambda z_{i-1}+(1-\lambda)z_i-\sigma(\tau_{i-1})} \\ &= \norm{\lambda z_{i-1}-\lambda\sigma(\tau_{i-1})+(1-\lambda)z_i-(1-\lambda)\sigma(\tau_{i-1})} \\ & \leq \norm{\lambda z_{i-1}-\lambda\sigma(\tau_{i-1})}+\norm{(1-\lambda)z_i-(1-\lambda)\sigma(\tau_{i-1})} \\ & = \lambda \norm{ z_{i-1}-\sigma(\tau_{i-1})}+(1-\lambda)\norm{z_i-\sigma(\tau_{i-1})},
\end{align*}
where the triangle inequality was used in the second-to-last transition. Due to Property~(i) and the fact that $\|z_i-\sigma(\tau_{i-1})\| \leq \beta + \rho$ it follows that 
\begin{align*}
    \norm{z-\sigma(\tau_{i-1})} &\leq  \lambda\beta+(1-\lambda)(\beta+\rho) \\ & = \beta+(1-\lambda)\rho \\ & \leq \beta+\rho.
\end{align*}

\vspace{5pt}

\noindent \emph{Property~(iv)}: This follows from properties~(i) and (ii), and triangle inequality: 

\begin{align*}
    \|z_i-z_{i-1}\| &\leq  \|z_i-\sigma(\tau_{i-1})\| + \|\sigma(\tau_{i-1})-z_{i-1}\| \\ & \leq 2\beta +\rho.
\end{align*}

\noindent \emph{Property~(v)}: Fix $1\leq i\leq \ell$. For any $\eta\in [0,1]$ define $u_\eta =\eta z_{i-1} + (1-\eta)z_i$. To prove that Property~(v) holds we will show that $u_\eta \in \C^f$ for every $\eta \in [0,1]$. 

Fix $\eta \in [0,1]$. Since $\sigma(\tau_{i-1}), \sigma(\tau_i)$ are points on a $\gamma$-clear path, the $\gamma$-balls $B_\gamma(\sigma(\tau_{i-1})), B_\gamma(\sigma(\tau_i))$ are collision-free. Thus it is sufficient to prove that $u_\eta \in B_\gamma(\sigma(\tau_{i-1})) \cup B_\gamma(\sigma(\tau_i))$. Next let $v_\eta$ be the closest point to $u_\eta$ on the straight line between $\sigma(\tau_{i-1})$ and $\sigma(\tau_i)$. That is,  $v_\eta=\argmin_{v \in \CH (\{\sigma(\tau_{i-1}), \sigma(\tau_i\})} \norm{u_\eta -v}$. 

We also define $v'_\eta:=\eta \sigma(\tau_{i-1}) + (1-\eta)\sigma(\tau_i)$, and prove that  $\|u_\eta-v'_\eta\|\leq \beta$. This implies that $\|u_\eta-v_\eta\|\leq \beta$ as $v_\eta$ is at least as close as $v'_\eta$ to $u_\eta$. Indeed, using Property~(i) and the triangle inequality we have
\begin{align*}
    &\|u_\eta-v'_\eta\|\\
    & \leq \norm{ \eta z_{i-1} + (1-\eta)z_i - (\eta \sigma(\tau_{i-1}) + (1-\eta)\sigma(\tau_i)) } \\ &\leq  \eta \norm{z_{i-1} - \sigma(\tau_{i-1})} + (1-\eta) \norm{z_i - \sigma(\tau_i)} = \beta.
\end{align*}

Building upon this result, we consider several cases with respect to the position of $v_\eta$. If $v_\eta$ is one of the endpoints, that is $v_\eta = \sigma(\tau_{i-1})$ or $v_\eta = \sigma(\tau_i)$, we have that $u_{\eta}$ is in a $\beta$-ball around this point. Since $\beta^2+(\rho/2)^2\leq \gamma^2$, and all are positive, we know that $\beta \leq \gamma$. Thus if $v_\eta$ is one of the endpoints we are done as $u_{\eta} \in B_\beta(\sigma(\tau_{i-1}))$ or $u_{\eta} \in B_\beta(\sigma(\tau_i))$.
    
Assuming $v_\eta$ is not one of the endpoints we can write $v_\eta = \eta^*\sigma(\tau_{i-1}) + (1-\eta^*)\sigma(\tau_i)$ where
\begin{align*}
\eta^* &= \arg\min_{\lambda \in (0,1)} \norm{ u_\eta - \lambda \sigma(\tau_{i-1}) + (1-\lambda) \sigma(\tau_i) }.
\end{align*}
In this case, the point $v_\eta$ is the projection of $u_\eta$ onto the line through $\sigma(\tau_{i-1})$ and $\sigma(\tau_i)$. Therefore,
$$\langle\sigma(\tau_i)-\sigma(\tau_{i-1}), u_\eta-v_\eta\rangle = 0.$$
Since $\{\sigma(\tau_{i-1}), \sigma(\tau_i),v_\eta\}$ are collinear we get that,
\begin{equation}
\begin{aligned}
\label{eq:inner_product_0}
& \langle v_\eta-\sigma(\tau_{i-1}), u_\eta-v_\eta\rangle = \\ & \langle v_\eta-\sigma(\tau_i), u_\eta-v_\eta\rangle = 0.
\end{aligned}
\end{equation}

Those equations are illustrated in Fig~\ref{fig:first_order_optimality}.

\begin{figure}[b]
  \centering
   \includegraphics[width=0.9\linewidth]{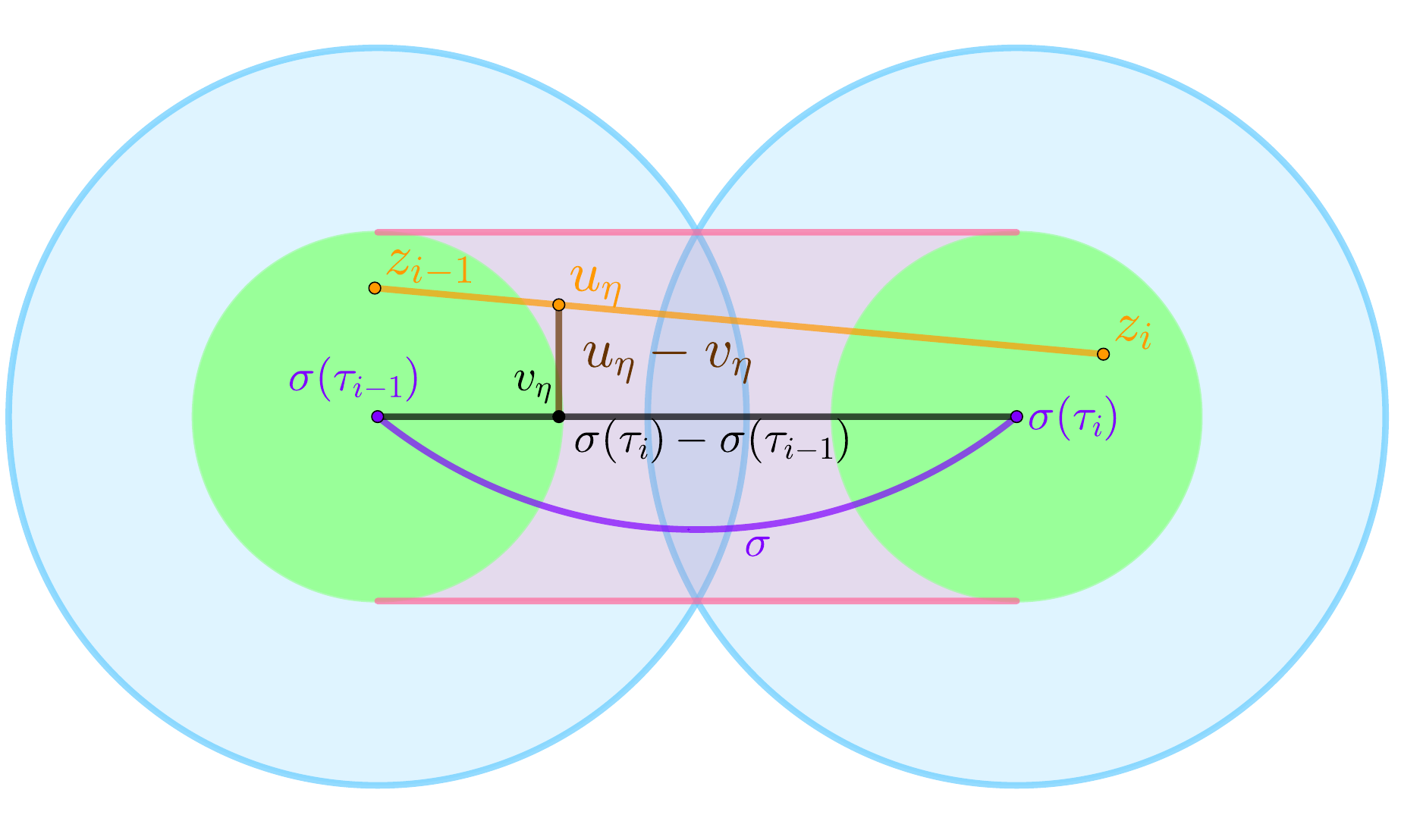}
  \caption{\sf{Illustration for the proof of Lemma~\ref{lem:super_duper}. The purple curve represents the trajectory $\sigma$. The green balls are the $\beta$-balls around $\sigma(\tau_{i-1})$ and $\sigma(\tau_i)$. The convex hull of the green balls (namely, the green balls together with the red region in-between) contains the line segment (in orange) $\CH \bigpar{ \bigbrace{z_{i-1}, z_i} }$, and we show that this region is collision-free.  The blue balls are collision-free $\gamma$-balls around $\sigma(\tau_{i-1})$ and $\sigma(\tau_i)$. The black line is $\CH \bigpar{ \bigbrace{\sigma(\tau_{i-1}), \sigma(\tau_i)} }$. 
   \label{fig:first_order_optimality}}}
\end{figure}

We now consider two separate cases with respect to $\eta^*$. First assume $\eta^* \in (0,1/2]$. Using Equation~\eqref{eq:inner_product_0} and Property~(i), it follows that 
\begin{align*}
    &\norm{u_\eta - \sigma(\tau_i)}^2 \\ & = \norm{(u_\eta - v_\eta) + (v_\eta - \sigma(\tau_i)}^2\\
    &= \norm{u_\eta - v_\eta}^2 + \norm{v_\eta - \sigma(\tau_i)}^2 \\ & \leq \beta^2 + \norm{\eta^* \sigma(\tau_{i-1}) + (1-\eta^*) \sigma(\tau_i) - \sigma(\tau_i)}^2\\
    &= \beta^2 + (\eta^*)^2 \norm{\sigma(\tau_{i-1})-\sigma(\tau_i)}^2 \\ & = \beta^2 + (\eta^*)^2 \rho^2\leq \beta^2 + (1/2)^2 \rho^2,
\end{align*}
where the last inequality follows from $\eta^* \leq 1/2$. Finally by plugging-in $\beta^2+(\rho/2)^2\leq \gamma^2$ we get
$\norm{u_\eta - \sigma(\tau_i)}^2\leq \gamma^2 $, which implies that $u_\eta \in B_\gamma(\sigma(\tau_i))$.

Due to symmetry, we can show that for $\eta^* \in (1/2,1)$, it follows that $\norm{u_\eta - \sigma(\tau_{i-1})}^2 \leq \gamma^2$, which again implies $u_\eta \in B_\gamma(\sigma(\tau_{i-1}))$.
To conclude, we showed that for all $\eta\in [0,1]$ it holds that  $u_\eta \in B_\gamma(\sigma(\tau_{i-1})) \cup  B_\gamma(\sigma(\tau_i))$, which implies that  $\CH\left(z_{i-1},z_i\right) \subset \C^f$.

\vspace{5pt}

\noindent \emph{Property~(vi):} Due to Property~(iv) we have that, $\|z_i-z_{i-1}\|\leq 2\beta +\rho$ for all $1\leq i \leq \ell$. 
By construction, our solution path through the $z_i$'s coincides with the starting and end points of $\sigma$, namely, $z_0 = \sigma(0), z_1 = \sigma(1)$. Hence, for the first and last segment of the motion we have $\|z_i-z_{i-1}\|\leq \beta +\|\sigma(\tau_i)-\sigma(\tau_{i-1})\|$. Using this we get,
    \begin{align*}
    &\sum_{1\leq i\leq \ell}\|z_i-z_{i-1}\|\\
    &= \|z_1-z_{0}\|
    + \|z_\ell-z_{\ell-1}\| +\sum_{2\leq i\leq \ell-1}\|z_i-z_{i-1}\| \\
    &\leq \beta +\rho+ \beta + \|\sigma(\tau_\ell)-\sigma(\tau_{\ell-1})\|+\sum_{2\leq i\leq \ell-1}\left(2\beta +\rho\right) \\
    &=  2\beta +\rho+ (\ell-2) \cdot (2\beta +\rho) + \|\sigma(\tau_\ell)-\sigma(\tau_{\ell-1})\|\\
    &=(\ell-1) \cdot (2\beta +\rho)+\|\sigma(\tau_\ell)-\sigma(\tau_{\ell-1})\|.
    \end{align*}
Additionally, due to equality $\|\sigma(\tau_i)-\sigma(\tau_{i-1})\|=\rho$ for all $1\leq i\leq \ell-1$, we have that 
\[\norm{\sigma} \geq (\ell-1) \cdot \rho+\|\sigma(\tau_\ell)-\sigma(\tau_{\ell-1})\|.\] Thus,
    \begin{align*}
    &\frac{\sum_{1\leq i\leq \ell}\|z_i-z_{i-1}\|}{\norm{\sigma}}\\
    &\leq \frac{(\ell-1) \cdot (2\beta +\rho)+\|\sigma(\tau_\ell)-\sigma(\tau_{\ell-1})\|}{(\ell-1) \cdot \rho+\|\sigma(\tau_\ell)-\sigma(\tau_{\ell-1})\|} \\
    &= \frac{(2\beta +\rho)}{\rho}\cdot\frac{(\ell-1)+\|\sigma(\tau_\ell)-\sigma(\tau_{\ell-1})\|/(2\beta +\rho)}{(\ell-1) +\|\sigma(\tau_\ell)-\sigma(\tau_{\ell-1})\|/\rho} \\
    &\leq \frac{(2\beta +\rho)}{\rho} =  1+\frac{2\beta}{\rho},
    \end{align*}
thus concluding this proof.
\ifijrrv
\qed
\fi
\end{proof}

It remains to prove Theorem~\ref{thm:main_single}. To use Lemma~\ref{lem:super_duper}, we set $\gamma=\delta,  \beta=\alpha \delta,  \rho={\frac{2}{\sqrt{1+\eps^2}}\delta},$ and $r={\frac{2(\eps+1)}{\sqrt{1+\eps^2}}\delta}$, for  $\alpha=\eps/\sqrt{1+\eps^2}$, so that the condition $\beta^2+(\rho/2)^2\leq \gamma^2$ holds. 

By setting $r=2\beta+\rho= {2(\alpha+\frac{1}{\sqrt{1+\eps^2}})\delta =\frac{2(\eps+1)}{\sqrt{1+\eps^2}}\delta}$, it follows from Lemma~\ref{lem:super_duper}~(iv) that for all $1\leq i\leq \ell$ the \prm edge connecting $z_i$ and $z_{i-1}$ is considered in the \prm construction. Additionally, since $\frac{2\beta}{\rho} = \eps$, it holds that $\norm{\widehat{\sigma}} \leq (1+\eps)\norm{\sigma}$ where  $\widehat{\sigma}$ denotes the piece-wise linear trajectory induced by $Z_\sigma^\rho$, which concludes the proof.

\subsection{Comparison with previous work}\label{sec:single_comparison}
In this section we compare our results with previous work, namely~\indirectcite{TsaoSoloveyETAL2020}. We first compare our result with their upper bound in the asymptotic case. In this comparison we observe an exponential improvement of order $1.3687^d$ in terms of sample size in favor to our method. We follow by comparing our result with their lower bound for $\eps=\infty$ (again in the asymptotic case). We observe that our result yields a sample set that is larger than the lower bound by an order of $1.4611^d$, which suggests that the lower-bound is not tight, or that smaller sample sets can be developed. Finally we provide a table which shows these comparisons for concrete values of $\delta$ and $\eps$.

First, we compare the minimal sample size required to guarantee ($\eps,\delta$)-completeness for the same single-robot setting, between the previous approach~\indirectcite{TsaoSoloveyETAL2020}, and the staggered grid as proposed here. 
We denote by $|\X_{\textup{prev}}|$ the size of the minimal sample set needed using the bound from \directcite{TsaoSoloveyETAL2020}, namely \[|\X_{\textup{prev}}| \approx \sqrt{\pi d} \bigpar{ \sqrt{\frac{2d}{\pi e}} \cdot \frac{ 1-(2-\alpha)\delta}{ \alpha \delta}}^d\;,\]
where, as above, $\alpha=\frac{\eps}{\sqrt{1+\eps^2}}$.
We denote by $|\X_{\textup{curr}}|=\abs{\X_{\alpha\delta,\delta}}$ the size of the minimal sample set needed using the bound for the staggered grid, as stated in Theorem~\ref{thm:main_single}, namely, \[{|\X_\textup{curr}}|= \left(\left\lceil\frac{(1-2\delta)\sqrt{d}}{\sqrt{8}\alpha \delta}\right\rceil\right)^d+\left(\left\lceil\frac{(1-2\delta)\sqrt{d}}{\sqrt{8}\alpha \delta}\right\rceil+1\right)^d\;.\]

We compare the quantities $|\X_{\textup{prev}}|$ and $|\X_{\textup{curr}}|$ for small values of $\eps$ and $\delta$ (we mention that the radius derived in~\directcite{TsaoSoloveyETAL2020} is equal to the one we are using here). We consider the ratio $\frac{|\X_{\textup{prev}}|}{|\X_{\textup{curr}}|}$ 
in the asymptotic regime where the values 
$\eps$ and $\delta$ tend to zero in the following equation (we ignore the rounding errors for both results as they do not change the asymptotic results):

\begin{equation}
\begin{aligned}
\label{eq:ratio}
\frac{|\X_{\textup{prev}}|}{|\X_{\textup{curr}}|} &=\frac{\sqrt{\pi d} \bigpar{ \sqrt{\frac{2d}{\pi e}} \cdot \frac{ 1-(2-\alpha)\delta}{ \alpha \delta}}^d}{\left(\frac{(1-2\delta)\sqrt{d}}{\sqrt{8}\alpha \delta}\right)^d+\left(\frac{(1-2\delta)\sqrt{d}}{\sqrt{8}\alpha \delta}+1\right)^d} \\
&= \frac{\sqrt{\pi d} \bigpar{ \sqrt{\frac{2d}{\pi e}} \cdot \frac{ 1-2\delta+\bm{\alpha\delta}}{ \alpha \delta}}^d}{\left(\frac{(1-2\delta)\sqrt{d}}{\sqrt{8}\alpha \delta}\right)^d+\left(\frac{(1-2\delta)\sqrt{d}}{\sqrt{8}\alpha \delta}+\bm{1}\right)^d}\\
&\overset{(a)}{\approx}
\frac{\sqrt{\pi d} \bigpar{ \sqrt{\frac{2d}{\pi e}} \cdot \frac{ 1-2\delta}{ \alpha \delta}}^d}{2\left(\frac{(1-2\delta)\sqrt{d}}{\sqrt{8}\alpha \delta}\right)^d} \\
&=\frac{\sqrt{\pi d}}{2} \bigpar{\sqrt{\frac{16}{\pi e}}}^d \approx \frac{\sqrt{\pi d}}{2}1.3687
^d\;,
\end{aligned}
\end{equation}
where the approximation in $(a)$ is due to ignoring the highlighted $\alpha\delta$ factor in the numerator of the previous expression, and also ignoring the highlighted unit value $1$ in the denominator (the influence of both values tends to $0$ as $\delta$ or $\eps$ tend to $0$). 
This implies that our result yields samples sets that are smaller by an exponential factor in $d$ than the previous work.

Next, we compare between the number of points we need in order to find a ($\infty,\delta$)-complete trajectory and the lower bound  derived in \indirectcite[Theorem 1]{TsaoSoloveyETAL2020} for this quantity, which was  only derived there for the case $\eps=\infty$. We denote the latter by $|\X_{\textup{LB}}|$, which is  proved in \directcite{TsaoSoloveyETAL2020} to be 
\begin{equation}
\label{eq:LB}
|\X_{\textup{LB}}| \approx \sqrt{\frac{e}{2}} \bigpar{1-\frac{2\delta}{1-2\delta}}^2\bigpar{ \sqrt{\frac{d-1}{2\pi e}} \cdot \frac{ 1-2\delta}{\delta}}^d\;.
\end{equation}

We compare the quantities $|\X_{\textup{LB}}|$ and $|\X_{\textup{curr}}|$ for small values of $\delta$ and using $\eps=\infty$. 
We consider the ratio $\frac{|\X_{\textup{curr}}|}{|\X_{\textup{LB}}|}$ 
in the asymptotic regime where the value of $\delta$ tends to zero in the following equation (we ignore the rounding errors for both results as they do not change the asymptotic results):

\begin{equation}
\begin{aligned}
\label{eq:curr_to_lb_ratio}
\frac{|\X_{\textup{curr}}|}{|\X_{\textup{LB}}|} &=\frac{\left(\frac{(1-2\delta)\sqrt{d}}{\sqrt{8}\alpha \delta}\right)^d+\left(\frac{(1-2\delta)\sqrt{d}}{\sqrt{8}\alpha \delta}+1\right)^d}{\sqrt{\frac{e}{2}} \bigpar{1-\frac{2\delta}{1-2\delta}}^2\bigpar{ \sqrt{\frac{d-1}{2\pi e}} \cdot \frac{ 1-2\delta}{\delta}}^d} \\
&=\frac{\left(\frac{(1-2\delta)\sqrt{d}}{\sqrt{8} \delta}\right)^d+\left(\frac{(1-2\delta)\sqrt{d}}{\sqrt{8} \delta}+\bm{1}\right)^d}{\sqrt{\frac{e}{2}} \bm{\bigpar{1-\frac{2\delta}{1-2\delta}}^2}\bigpar{ \sqrt{\frac{d-1}{2\pi e}} \cdot \frac{ 1-2\delta}{\delta}}^d} \\
&\overset{(a)}{\approx}
\frac{2\left(\frac{(1-2\delta)\sqrt{d}}{\sqrt{8}\alpha \delta}\right)^d}{\sqrt{\frac{e}{2}}\bigpar{ \sqrt{\frac{d-1}{2\pi e}} \cdot \frac{ 1-2\delta}{\delta}}^d} \\
&=\sqrt{\frac{8}{e}}\bigpar{\sqrt{\frac{\pi e}{4}}}^d\sqrt{\bigpar{\frac{d}{d-1}}^d}\\
&\approx \sqrt{\frac{8}{e}}\sqrt{\bigpar{\frac{d}{d-1}}^d} \cdot 1.4611^d\;,
\end{aligned}
\end{equation}
where the approximation in $(a)$ is due to ignoring the highlighted unit value $1$ in the numerator of the previous expression, whose influence tends to $0$ as $\delta$ tend to $0$, and also ignoring the highlighted squared value in the denominator, whose value tends to $1$ as $\delta$ tend to $0$. 

Finally, by noticing that for $d \geq 2$ it holds that $$e\leq \bigpar{\frac{d}{d-1}}^d \leq 4,$$
it follows that 
$$2\sqrt{2} \cdot 1.4611^d \leq \frac{|\X_{\textup{curr}}|}{|\X_{\textup{LB}}|} \leq 4\sqrt{\frac{2}{e}} \cdot 1.4611^d.$$

This shows that our upper bound is at most a factor of roughly $1.4611^d$ away from the lower bound.
\ifextendstaggered
\dror{New text here.}
These results will be interesting when compared with the same results on the coverage of the unit cube in Section~\ref{sec:staggered_grid_analysis}.
\fi

To conclude this section, we provide in Table~\ref{tbl:samp_comp_comparison} a comparison of sufficient and necessary samples sizes for $(\eps,\delta)$-completeness for various values of the clearance $\delta$ and the stretch $\eps$ in dimensions $d\in\{2,\ldots,6\}$. In particular, for each combination of parameter values we report the following three quantities:  (I) A lower bound (Eq.~\ref{eq:LB}) on the \emph{necessary} number of samples for $(\eps,\delta)$-completeness, as was derived in~\directcite[Theorem~1]{TsaoSoloveyETAL2020}, which applies only to the case $\eps=\infty$. (II) The size of our sampling distribution, i.e., $|\X_{\textup{curr}}|$, which guarantees  $(\eps,\delta)$-completeness, according to Theorem~\ref{thm:main_single}. (III) An upper bound on the \emph{sufficient} number of samples for $(\eps,\delta)$-completeness, denoted by $|\X_{\textup{prev}}|$ above, as was derived in~\directcite[Theorem~2]{TsaoSoloveyETAL2020}. 

\begin{table*}[t]
\footnotesize
\centerline{

\begin{tabular}{|c|c|c||c|c||c|c||c|c||c|c|}
\cline{3-11}
                      \multicolumn{2}{c|}{}   & \multicolumn{3}{c||}{$\eps=\infty$}       & \multicolumn{2}{c||}{$\eps=1$} & \multicolumn{2}{c||}{$\eps=0.25$} & \multicolumn{2}{c|}{$\eps=0.1$} \\ \hline
                    $\delta$  &  $d$ & lower bound ($|\X_{\textup{LB}}|$)      & $|\X_{\textup{curr}}|$      & $|\X_{\textup{prev}}|$     & $|\X_{\textup{curr}}|$        & $|\X_{\textup{prev}}|$      & $|\X_{\textup{curr}}|$         & $|\X_{\textup{prev}}|$        & $|\X_{\textup{curr}}|$         & $|\X_{\textup{prev}}|$       \\ \hline \hline
\multirow{5}{*}{$0.25$} & $2$ & $0       $ & $5       $ & $12      $ & $13      $   & $19      $  & $61      $ & $109     $ & $265     $ & $567     $ \\ \cline{2-11} 
                      & $3$ & $0       $ & $35      $ & $52      $ & $35      $   & $108     $  & $559     $ & $1510    $ & $4941    $ & $1.79\cdot 10^{4}$ \\ \cline{2-11} 
                      & $4$ & $0       $ & $97      $ & $263     $ & $97      $   & $697     $  & $3697    $ & $2.37\cdot 10^{4}$ & $1.16\cdot 10^{5}$ & $6.43\cdot 10^{5}$ \\ \cline{2-11} 
                      & $5$ & $0       $ & $275     $ & $1478    $ & $1267    $   & $5000    $  & $4.96\cdot 10^{4}$ & $4.11\cdot 10^{5}$ & $2.47\cdot 10^{6}$ & $2.54\cdot 10^{7}$ \\ \cline{2-11} 
                      & $6$ & $0       $ & $793     $ & $9029    $ & $4825    $   & $3.90\cdot 10^{4}$  & $7.94\cdot 10^{5}$ & $7.74\cdot 10^{6}$ & $8.11\cdot 10^{7}$ & $1.09\cdot 10^{9}$ \\ \hline \hline
\multirow{5}{*}{$0.1$}  & $2$ & $3       $ & $61      $ & $104     $ & $85      $   & $194     $  & $613     $ & $1471    $ & $3445    $ & $8437    $ \\ \cline{2-11} 
                      & $3$ & $15      $ & $341     $ & $1393    $ & $855     $   & $3566    $  & $1.99\cdot 10^{4}$ & $7.50\cdot 10^{4}$ & $2.58\cdot 10^{5}$ & $1.03\cdot 10^{6}$ \\ \cline{2-11} 
                      & $4$ & $88      $ & $3697    $ & $2.13\cdot 10^{4}$ & $1.07\cdot 10^{4}$   & $7.45\cdot 10^{4}$  & $7.22\cdot 10^{5}$ & $4.33\cdot 10^{6}$ & $2.19\cdot 10^{7}$ & $1.42\cdot 10^{8}$ \\ \cline{2-11} 
                      & $5$ & $595     $ & $4.96\cdot 10^{4}$ & $3.59\cdot 10^{5}$ & $1.59\cdot 10^{5}$   & $1.72\cdot 10^{6}$  & $3.16\cdot 10^{7}$ & $2.76\cdot 10^{8}$ & $2.23\cdot 10^{9}$ & $2.17\cdot 10^{10}$ \\ \cline{2-11} 
                      & $6$ & $4459    $ & $3.80\cdot 10^{5}$ & $6.58\cdot 10^{6}$ & $2.77\cdot 10^{6}$   & $4.32\cdot 10^{7}$  & $1.32\cdot 10^{9}$ & $1.91\cdot 10^{10}$ & $2.46\cdot 10^{11}$ & $3.60\cdot 10^{12}$ \\ \hline \hline
\multirow{5}{*}{$0.05$} & $2$ & $21      $ & $221     $ & $460     $ & $365     $   & $892     $  & $2965    $ & $7204    $ & $1.67\cdot 10^{4}$ & $4.21\cdot 10^{4}$ \\ \cline{2-11} 
                      & $3$ & $234     $ & $3925    $ & $1.31\cdot 10^{4}$ & $9009$   & $3.54\cdot 10^{4}$  & $2.01\cdot 10^{5}$ & $8.13\cdot 10^{5}$ & $2.77\cdot 10^{6}$ & $1.15\cdot 10^{7}$ \\ \cline{2-11} 
                      & $4$ & $3152    $ & $6.70\cdot 10^{4}$ & $4.23\cdot 10^{5}$ & $2.35\cdot 10^{5}$   & $1.59\cdot 10^{6}$  & $1.64\cdot 10^{7}$ & $1.04\cdot 10^{8}$ & $5.45\cdot 10^{8}$ & $3.55\cdot 10^{9}$ \\ \cline{2-11} 
                      & $5$ & $4.82\cdot 10^{4}$ & $1.81\cdot 10^{6}$ & $1.51\cdot 10^{7}$ & $9.24\cdot 10^{6}$   & $7.88\cdot 10^{7}$  & $1.49\cdot 10^{9}$ & $1.46\cdot 10^{10}$ & $1.26\cdot 10^{11}$ & $1.21\cdot 10^{12}$ \\ \cline{2-11} 
                      & $6$ & $8.13\cdot 10^{5}$ & $4.09\cdot 10^{7}$ & $5.83\cdot 10^{8}$ & $3.39\cdot 10^{8}$   & $4.25\cdot 10^{9}$  & $1.58\cdot 10^{11}$ & $2.24\cdot 10^{12}$ & $3.05\cdot 10^{13}$ & $4.49\cdot 10^{14}$ \\ \hline \hline
\multirow{5}{*}{$0.01$} & $2$ & $734     $ & $5101    $ & $1.25\cdot 10^{4}$ & $9941$   & $2.48\cdot 10^{4}$  & $8.28\cdot 10^{4}$ & $2.09\cdot 10^{5}$ & $4.87\cdot 10^{5}$ & $1.24\cdot 10^{6}$ \\ \cline{2-11} 
                      & $3$ & $4.58\cdot 10^{4}$ & $4.65\cdot 10^{5}$ & $1.85\cdot 10^{6}$ & $1.25\cdot 10^{6}$   & $5.20\cdot 10^{6}$  & $3.07\cdot 10^{7}$ & $1.27\cdot 10^{8}$ & $4.42\cdot 10^{8}$ & $1.83\cdot 10^{9}$ \\ \cline{2-11} 
                      & $4$ & $3.36\cdot 10^{6}$ & $4.94\cdot 10^{7}$ & $3.11\cdot 10^{8}$ & $1.88\cdot 10^{8}$   & $1.23\cdot 10^{9}$  & $1.35\cdot 10^{10}$ & $8.73\cdot 10^{10}$ & $4.73\cdot 10^{11}$ & $3.06\cdot 10^{12}$ \\ \cline{2-11} 
                      & $5$ & $2.80\cdot 10^{8}$ & $5.96\cdot 10^{9}$ & $5.78\cdot 10^{10}$ & $3.30\cdot 10^{10}$   & $3.22\cdot 10^{11}$  & $6.76\cdot 10^{12}$ & $6.63\cdot 10^{13}$ & $5.76\cdot 10^{14}$ & $5.66\cdot 10^{15}$ \\ \cline{2-11} 
                      & $6$ & $2.57\cdot 10^{10}$ & $7.82\cdot 10^{11}$ & $1.17\cdot 10^{13}$ & $6.44\cdot 10^{12}$   & $9.16\cdot 10^{13}$  & $3.71\cdot 10^{15}$ & $5.47\cdot 10^{16}$ & $7.73\cdot 10^{17}$ & $1.14\cdot 10^{19}$ \\ \hline
\end{tabular}
}
\caption{\sf{Sufficient and necessary samples sizes for $(\eps,\delta)$-completeness. A comparison of the specific sample sizes for various values for the attributes of clearance $\delta$, stretch $\eps$, and dimension $d$. The column ``lower bound'' reports the value from~\indirectcite[Theorem~1]{TsaoSoloveyETAL2020}. For every combination of attributes we report the value $|\X_{\textup{curr}}|$  from Theorem~\ref{thm:main_single}, and $|\X_{\textup{prev}}|$ from~\indirectcite[Theorem~2]{TsaoSoloveyETAL2020}.} Notice that it always holds that $|\X_{\textup{curr}}|<|\X_{\textup{prev}}|$}.\label{tbl:samp_comp_comparison}
\end{table*}

As reported in Theorem~\ref{thm:main_single} and~\indirectcite[Theorem~2]{TsaoSoloveyETAL2020}, both quantities $|\X_{\textup{curr}}|$ and $|\X_{\textup{prev}}|$ increase exponentially as the dimension increases, and as the clearance $\delta$ or the stretch $\eps$ decrease. However, the value $|\X_{\textup{curr}}|$ is consistently smaller than $|\X_{\textup{prev}}|$. As the latter value increases it grows more rapidly than $|\X_{\textup{curr}}|$ by at least one order of magnitude. Although  in certain cases the value of $|\X_{\textup{curr}}|$  is still quite large for some practical application (particularly when the dimension is larger than $3$), we hope that our work would motivate further study into even smaller sampling distributions than our staggered grid, or a more refined version of Theorem~\ref{thm:main_single}. In this context, we point out that the value $|\X_{\textup{curr}}|$ is larger than the lower bound (where applicable) by one order of magnitude, which suggests that  there is room for improvement.


\fi

\section{Near-optimal tensor roadmaps for MRMP}\label{sec:multi_main}
We present our central contribution: we extend our results from the previous section to the multi-robot setting. In particular, we prescribe the number of samples and their distribution with respect to each robot's \prm graph, such that the resulting tensor roadmap encodes a solution for the multi-robot motion planning (MRMP) problem that has a desired near-optimality guarantee.

\subsection{Basics of multi-robot motion planning}\label{sec:multi_basic}
We provide a definition of the MRMP  problem. We consider the setting of $R \geq 2$ identical 
robots operating in a shared workspace, and denote by $\C_i \subset [0,1]^d$ the configuration space of robot $i$, $1 \leq i \leq R$. We define $\C_i^f, \C_i^o\subset \C_i$ to be the free and forbidden spaces, respectively, of robot $i$. Since the robots are identical it holds that $\C_i^f = \C_j^f$ for all $1\leq i \leq j \leq R$.

The configuration space of the multi-robot system $\dC$, termed \emph{the composite configuration space}, is the Cartesian product of the individual robots' configuration spaces, i.e., $\dC=\C_1\times \ldots\times \C_R$. That is, a composite configuration $Q = (q_1,...,q_R)\in \dC$ is an $R$-tuple of single-robot configurations, where $q_i\in \C_i$.
For two distinct robots $i, j$, we denote by $I_i^j(q_j) \subset \C_i$ the set of configurations of robot $i$ that lead to collision with robot $j$ when $j$ is at configuration $q_j$. The composite free space $\dC^f \subset \dC$ consists of all composite configurations $(q_1,...,q_R)$ such that (i) $q_i \in \C_i^f$ for every $1 \leq i \leq R$, and (ii) $q_i \notin I_i^j(q_j)$ for every $1 \leq i\neq j \leq R$, which ensure that robot-obstacle and robot-robot collisions are avoided, respectively. 

Given start and goal positions $x^s_i,x^g_i\in \C^f_i$, respectively, for each robot $1\leq i\leq R$, let $\vec{x}^{s}=(x^s_1,\ldots,x^s_R)$ and $\vec{x}^{g}=(x^g_1,\ldots,x^g_R)$. The MRMP problem, denoted by $\dM=(\dC^f, \vec{x}^{s},  \vec{x}^{g})$, consists of finding trajectories for the $R$ robots such that the robots begin their motion at $\vec{x}^{s}$, end at $\vec{x}^{g}$, and avoid collisions (both with obstacles and with each other) along the way. Formally, the objective is to find a collision-free composite trajectory of the form $\Sigma : [0,1] \to \dC^f$, where $\Sigma$ is an $R$-tuple $\Sigma = (\sigma_1,...,\sigma_R)$ of single-robot trajectories $\sigma_i:[0,1] \to \C_i^f$, such that $\Sigma(0) = \vec{x}^{s},\Sigma(1) = \vec{x}^{g}$, and $\Sigma(\tau)\in\dC^f$ for all $0\leq\tau\leq 1$. 

In this work we are interested in finding high-quality solutions for the multi-robot problem. We consider as cost criterion the sum of single-robot trajectory lengths, denoted as $\cost(\Sigma) = \sum_{i=1}^R \norm{\sigma_i}$. Our analysis applies more generally to cost functions having the following property. Let $\Sigma^1=(\sigma^1_1,\ldots,\sigma^1_R), \Sigma^2=(\sigma^2_1,\ldots,\sigma^2_R)$ be multi-robot trajectories, such that $\norm{\sigma^1_i} \leq (1+\eps)\norm{\sigma^2_i}$ for all $i$,  $1 \leq i \leq R$. Our analysis will work for any cost function $c$ for which the condition above for the individual robots implies $c(\Sigma^1) \leq (1+\eps)c(\Sigma^2)$. For example, our analysis can be easily adapted to the case where the cost criterion is  the maximum of trajectory lengths, i.e., $\max_{i=1}^R \norm{\sigma_i}$
\ifextendedv
.
\fi

\subsection{Tensor roadmaps}
We provide a formal definition of the tensor roadmap (TR), which is implicitly explored by sampling-based planners such as \drrt~\indirectcite{drrt.doi:10.1177/0278364915615688}, \drrtstar~\indirectcite{DBLP:journals/arobots/ShomeSDHB20}, as well as by search-based methods such as \mccbs~\indirectcite{LiEA2019} and \mstar\indirectcite{WagCho15}.

For every robot $i$, $1 \leq i \leq R $, let $G_i(\X_i, r_i) = (V_i,E_i)$ be a \prm graph embedded in $\C^f_i$, for some point set $\X_i$ and radius $r_i>0$ (as defined in Section~\ref{ssec:roadmaps-and-samples}). The TR, denoted by ${\widehat{G}(\vec{\X}, \vec{r})}$ = $(\widehat{V}, \widehat{E})$, is the tensor product of $G_1,\ldots,G_R$. In particular, each vertex of ${\widehat{G}(\vec{\X}, \vec{r})}$ describes a simultaneous placement of the $R$ robots, and similarly an edge of ${\widehat{G}(\vec{\X}, \vec{r})}$ describes a simultaneous motion of the robots. Formally,
\begin{enumerate}[label=(\roman*)]
  \item $\widehat{V} = \{(v_1,\ldots,v_R):\forall i, v_i\in V_i\}$, and
  \item for two vertices $W=(w_1,\ldots,w_R), U=(u_1,\ldots,u_R) \in \widehat{V}$, the edge set $\widehat{E}$ contains the edge $(W,U)$ if for all $i$, $w_i=u_i$ or $(w_i,u_i)\in E_i$.
\end{enumerate}

Note that robots are allowed to stay put, which differs from prevalent definitions of the tensor product of graphs~\indirectcite{10.5555/2031398,BALAKRISHNAN19981,Shitov_2019}. 
Notice further that by the definition of $G_i$, the motion described by each edge in $\widehat{E}$ represents a trajectory for the $R$ robots in which the robot-obstacle collisions are avoided. Next we consider a subgraph of the TR in which also robot-robot collisions are avoided. 
Given an MRMP problem, $\dM= (\dC^f, \vec{x^{s}},  \vec{x^{g}})$ we will denote as $\widehat{G}_{\dM(\vec{\X}, \vec{r})}(x^{s}, x^{g})$ its TR constructed from the \prm graphs of the individual robots' ${G_{\M_i(\X_i, r_i)}}$, but where \emph{we remove all edges of the TR representing transitions of the robots that are not collision free}. We denote by ${d({\widehat{G}_{\dM(\vec{\X}, \vec{r})}},\vec{x^{s}},  \vec{x^{g}})}$ the minimal cost of collision-free trajectories from $\vec{x^{s}}$ to $\vec{x^{g}}$ in the graph $\widehat{G}_{\dM(\vec{\X}, \vec{r})}$.

\subsection{Multi-robot clearance and completeness}\label{sec:multi completeness}
In preparation for defining the multi-robot equivalent of $(\eps,\delta)$-completeness, we first define clearance for the multi-robot case. 
Given a trajectory $\Sigma$, recall that we define for each robot its forbidden space at time $\tau\in [0,1]$ to be its obstacle space $\C_i^o$ and the configurations that will lead to collisions with other robots. Formally, $\C_i^o(\tau) = \C_i^o \bigcup_{j\neq i}I_i^j(\sigma_j(\tau))$ is the forbidden space for robot $i$ at time $\tau$. Notice that we define $\C_i^o(\tau)$ only for a given trajectory $\Sigma$ as it depends on the locations of the other robots.

\begin{definition}[$\vec{\delta}$-clearance]
Given a trajectory $\Sigma$, we say that $\Sigma$ has $\vec{\delta}$-clearance for $\vec{\delta} = (\delta_1,\ldots, \delta_R)$ if for each robot $i$, $1\leq i \leq R$ and at any time $\tau\in [0,1]$, the distance from $i$ to the obstacles and to each robot $j\neq i$ is at least $\delta_i$. Formally, $\Sigma$ has $\vec{\delta}$-clearance if, for all $1\leq i \leq R$, $0\leq \tau \leq 1$, it holds that $\norm{\sigma_i(\tau)-x}> \delta_i$, for every $x \in \C_i^o(\tau)$.
\end{definition}

Next we define the equivalent of single-robot ($\eps,\delta$)-completeness for MRMP:

\begin{definition}[Multi-robot ($\eps,\vec{\delta}$)-completeness]
Given $R$ robots, a stretch parameter $\eps>0$, a vector of $R$ sample sets $\vec{\X}=(\X_1,\ldots,\X_R)$, and a vector of $R$ connection radii $\vec{r}=(r_1,\ldots,r_R)$, we say that the pair $(\vec{\X}, \vec{r})$ is ($\eps,\vec{\delta}$)-complete if for every $\vec{\delta}$-clear $\dM= (\dC^f, \vec{x^{s}},  \vec{x^{g}})$ it holds that
$$
{d({\widehat{G}_{\dM(\vec{\X}, \vec{r})}},\vec{x^{s}},  \vec{x^{g}})}\leq (1+\eps)\OPT_{\vec{\delta}},
$$
where $\OPT_{\vec{\delta}}$ is the minimal cost of a $\vec{\delta}$-clear $\dM$ solution.
\end{definition}



We are ready to state our main contribution.

\begin{theorem}[Sufficient conditions for MRMP ($\eps,\vec{\delta}$)-completeness]\label{thm:multi_robot}
Let $\eps>0$ be a stretch factor, let $\vec{\delta}$ be a clearance vector $(\delta_1,\ldots,\delta_R)$, and denote $\omega=\frac{\eps}{2(\eps+2)}$. Define the sampling distributions $\vec{\X} = (\X_1,\ldots,\X_R)$ and radii vector $\vec{r} = (r_1,\ldots,r_R)$, as 
\[\X_i=\X_{\omega\delta_i,\delta_i}, \quad r_i={{\delta_i(\eps+1)}/{(\eps+2)}},\] for every robot $1\leq i\leq R$. Then $\left(\vec{\X}, \vec{r}\right)$ is $\left(\eps,\vec{\delta}\right)$-complete.
\end{theorem}

\ifextendedv
\begin{proof}
Fix $\eps>0$ and $\vec{\delta}=(\delta_1,\ldots,\delta_R)$ where $\delta_i>0$ for all $1\leq i\leq R$. Let $\vec{x}^{s}, \vec{x}^{g}$ be the start and goal composite configurations, respectively. Let $\dM=(\dC^f, \vec{x}^{s}, \vec{x}^{g})$ be a $\vec{\delta}$-clear MRMP problem. (This implies in particular that the start and goal configurations fulfill the clearance requirements.) Let $\Sigma^*=(\sigma^*_1,\ldots,\sigma^*_R)$ be a  $\vec{\delta}$-clear solution which minimizes the expression $\cost(\Sigma)$ over all  $\vec{\delta}$-clear solution trajectories $\Sigma$.

The overall structure of this proof is as follows. First, we generate for each robot $i$ a \prm graph $G_i$, which is induced by the sample set $\X_i$ and  the radius $r_i$. We show, using  Lemma~\ref{lem:super_duper}, that each such  $G_i$ induces a collision-free trajectory $\bar{\sigma}_i$ for robot $i$, which closely follows $\sigma^*_i$. Next we show, using additional properties following from  Lemma~\ref{lem:super_duper}, that we can specify the arrival time of the individual robots along their corresponding vertices along $\bar{\sigma}_i$, producing the path $\widehat{\sigma_i}$  such that robot-robot collisions are avoided. This implies that the induced tensor graph contains a collision-free composite trajectory $\widehat{\Sigma}=(\widehat{\sigma}_1,\ldots,\widehat{\sigma}_R)$, whose cost is at most $(1+\eps)\cost(\Sigma^*)$. We fill in the details of the proof below.


For every $1\leq i\leq R$, define $\X_i=\X_{\beta_i,\delta_i}$, where 
$\beta_i= \omega\delta_i$. Let $G_i = (V_i,E_i)$ be the \prm graph for robot~$i$ using $(\X_i,r_i)$, that is $G_i = G_{\M_i(\X_i, r_i)}$ for $\M_i = (\C_i,x_i^s,x_i^g)$. Let $\widehat{G}$ be the  tensor product of $G_1,\ldots,G_R$.
Also, define $\rho_i=\frac{\delta_i}{\eps+2}$ and let $T_i:=T_{\sigma^*_i}^{\rho_i}=(\tau^i_0,\ldots,\tau_{\ell_i}^i)$, as in Definition~\ref{def:trajectory-timing}. First, we show that the conditions for Lemma~\ref{lem:super_duper} hold for each robot $i$, i.e., $\beta_i^2+(\rho_i/2)^2\leq \delta_i^2$. Indeed,
\begin{align*}
& \beta_i^2+\left(\frac{\rho_i}{2}\right)^2 =  \omega^2\delta_i^2+\frac{\delta_i^2}{4(\eps+2)^2} \\
& = \left(\frac{\eps^2+1}{(\eps^2+1)+(3\eps^2+8\eps+15)}\right)\delta_i^2 
\leq \delta_i^2. 
\end{align*}
Thus, we can apply Lemma~\ref{lem:super_duper} with the time sequence $T_i$ for each robot $i$ individually to show that there exists a point set $Z_i:=Z_{\sigma^*_i}^{\rho_i}=(z_0^i,\ldots,z_{\ell_i}^i)\subset \X_i\cup\{x_i^s,x_i^g\}$, such that Properties~(i)-(vi) of Lemma~\ref{lem:super_duper} hold. In particular, we have that for all $1\leq i \leq R$,
\begin{enumerate}[label=(\roman*')]
    \item $\|z_k^i-\sigma^*_i(\tau_k^i)\|\leq \beta_i$, for all $0\leq k\leq \ell_i$;
    \item $\|z-\sigma^*_i(\tau_{k-1}^i)\|\leq \beta_i +\rho_i$, for all $1\leq k\leq \ell_i$, $z \in \CH\left(z_{k-1}^i,z_k^i\right)$;
    \item $\|z-\sigma^*_i(\tau_k^i)\|\leq \beta_i +\rho_i$, for all $1\leq k\leq \ell_i$, $z \in \CH\left(z_{k-1}^i,z_k^i\right)$;
    \item $\|z_k^i-z_{k-1}^i\|\leq 2\beta_i +\rho_i$, for all $1\leq k\leq \ell_i$;
    \item $\CH\left(z_{k-1}^i,z_k^i\right)$ is collision free, for all $1\leq k\leq \ell_i$. That is $\CH\left(z_{k-1}^i,z_k^i\right) \subset \C^f_i$, for all $1\leq k\leq \ell_i$;
    \item $\sum_{1\leq k\leq \ell_i}\|z_k^i-z_{k-1}^i\|\leq \left(1+\frac{2\beta_i}{\rho_i}\right)\norm{\sigma^*_i}$.
\end{enumerate}

By setting $r_i=2\beta_i+\rho_i = \frac{\delta_i(\eps+1)}{(\eps+2)}$, and using Property~(iv'), it follows that the edge connecting $z_j^i$ and $z_{j-1}^i$ is considered in the construction of $G_i$, and from (v'), it is in $G_i$. Next, due to the fact that $2\beta_i/\rho_i = \eps$
, and Property~(vi'), it holds that $\norm{\bar{\sigma}_i} \leq (1+\eps)\norm{\sigma^*_i}$, for $\bar{\sigma}_i$ denoting the trajectory induced by $Z_{\sigma^*_i}^{\rho_i}$. Thus, each robot has a collision free path, which is a $(1+\eps)$-approximation for its $\delta_i$-clear path $\sigma^*_i$. Finally, observe that $\cost(\bar{\Sigma})\leq (1+\eps)\cost(\Sigma^*)$, where $\bar{\Sigma}=(\bar{\sigma}_1,\ldots,\bar{\sigma}_R)$. Note that this last step also applies to a cost function that returns the maximum length over the $R$ single-robot trajectories.



As robots may collide with one another along the paths $\bar{\Sigma}$, we leverage the above properties (i')-(iv') to show that $\widehat{G}$, the TR where edges describe coordinated collision-free motion where the robots avoid obstacles as well as each other, contains a high-quality composite trajectory which avoids robot-obstacle \emph{and robot-robot} collisions. To do so, we show that we can adjust the positions of the robots along the trajectories $\bar{\sigma}_1,\ldots,\bar{\sigma}_R$, to induce a collision-free trajectory over $\widehat{G}$, which we denote by $\widehat{\Sigma}=(\widehat{\sigma}_1,\ldots,\widehat{\sigma}_R)$.

First, define a list $\L$ of triplets of the following form:
$$\L:=\bigcup_{i=1}^R\bigcup_{j=1}^{\ell_i}\{(i,\tau^i_j,z_{j}^i)\}.$$ 
That is, $\L$ contains for every robot $1\leq i\leq R$, $\ell_i$ triplets of the form $(i,\tau^i_j,z_{j}^i)$, where $\tau^i_j\in T_i$ is a timestamp, and $z_{j}^i\in Z_i$ is the corresponding configuration. Additionally, define $\L^o$ to be a permutation of $\L$, where the triplets are ordered according to the timestamp.
That is,
$$\L^o:=\{(i_1,\tau_{i_1},z_{i_1}), (i_2,\tau_{i_2},z_{i_2}),\ldots, (i_\ell,\tau_{i_\ell},z_{i_\ell})\},$$
where $\ell=\sum_{i=1}^R \ell_i$, $(i_j,\tau_{i_j},z_{i_j})\in \L$ for every $1\leq j\leq \ell$, and $\tau_{i_j}\leq \tau_{i_{j+1}}$ for every $1\leq j\leq \ell-1$. For simplicity, when several robots have the same timestamp, we order the corresponding triplets in increasing order of the robot indices.

Next, we describe an iterative scheme that uses $\L^o$ for generating a sequence of composite vertices $V_0, V_1,\ldots, V_\ell\in \widehat{V},$ such that $(V_j,V_{j+1})\in \widehat{E}$. First, define $V_0=\vec{x}^s$. Next, given that $V_j=(v_{j_1},\ldots,v_{j_R})$ has already been defined for some $1\leq j\leq \ell-1$, set $V_{j+1}=(v_{(j+1)_1},\ldots,v_{(j+1)_R})$, where $v_{(j+1)_{i_{j+1}}}:=z_{i_{j+1}}$, and $v_{(j+1)_{i'}}=v_{j_{i'}}$ for every $i'\neq i_{j+1}$. Namely, when transitioning from $V_j$ to $V_{j+1}$ all the robots stay put, besides robot $i_{j+1}$ whose timestamp appeared in item $j+1$ of $\L^o$. See appendix for a concrete example.

To complete the proof, we first note that $(V_j,V_{j+1})\in \widehat{E}$ for every  $1\leq j\leq \ell-1$. This follows from the values of the connection radii $\vec{r}$ we assigned, and Properties~(iv') and~(v'), which ensure that $\{v^i_j,v^i_{j+1}\}\in E_i$ for every robot $1\leq i\leq R$, where $(v_1,\ldots,v_R):=V_j,(v'_1,\ldots,v'_R):=V_{j+1}$.


It remains to prove that the robots do not collide with one another while they move along the path represented by any such edge $(V_j,V_{j+1})$.  
First, recall that there is exactly one robot moving for $(V_j,V_{j+1})$. In particular, this is the robot whose index is $i_j$, which is the first value of the $j$th triplet  $(i_j,\tau_{i_j},z_{i_j})\in \L^o$. That is $v_{i_j}\neq v'_{i_j}$, whereas for any other robot $k\neq i_j$ it holds that $v_{k}= v'_{k}$.

Notice that given two stationary robots $k_1,k_2\neq i_j$, and assuming that they did not collide for the previous edge $(V_{j-1},V_{j})$, they do not collide with each other for $(V_j,V_{j+1})$ as well. 

It remains to show that robot $i_j$ does not collide with any stationary robot $k\neq i_j$, while moving from $v_{i_j}$ to $v'_{i_j}$. By definition of $V_0,\ldots,V_\ell$, it must be that robot $k$ resides in $v_{k}=z^{k}_{j'}\in Z_{k}$ such that $\tau^{k}_{j'}\leq \tau_{i_j}\leq \tau^{k}_{j'+1}$, for some $0\leq j'\leq \ell_{k}-1$.
Next, recall that $\sigma^*_{k},\sigma^*_{i_j}$ are $\delta_{k}$ and $\delta_{i_j}$-clear, respectively. Thus, it suffices to prove that $\|\sigma^*_{i_j}(\tau_{i_j})-p\|\leq \delta_{i_j}/2$ for every $p\in \CH\left(v_{i_j}, v'_{i_j}\right)$ and $\|\sigma^*_{k}(\tau_{i_j})-z^{k}_{j'}\|\leq \delta_{k}/2$.


Indeed, Property~(iii') implies that for $p\in \CH\left(v_{i_j}, v'_{i_j}\right)$ we have that 
$\norm{p-\sigma^*_{i_j}(\tau_{i_j})}\leq \beta_{i_j} +\rho_{i_j}\leq \delta_{i_j}/2$.
We also have that 
\begin{align*}\label{eq:i_dist_2}
\norm{z^{k}_{j'}-\sigma^*_{k}(\tau_{i_j})}& \leq 
\|z^{k}_{j'}-\sigma^*_{k}(\tau^{k}_{j'})\|+\|\sigma^*_{k}(\tau^{k}_{j'})-\sigma^*_{k}(\tau_{i_j})\| \\
& \leq \beta_{k} + \rho_{k}\leq \delta_{k}/2.
\end{align*}
Since at time $\tau_{i_j}$ both robots $i_j$ and $k$ obey the clearance rules, they are at distance $\delta_{\rm max}=\max(\delta_{i_j},\delta_k)$ from one another. As we have shown that during the specified motion, robot $i_j$ has not drifted more than $\delta_{\rm max}/2$ from its position at time $\tau_{i_j}$, and the stationary placement of robot $k$ during this motion is at most $\delta_{\rm max}/2$ from its placement at time $\tau_{i_j}$, we are guaranteed that they do not collide with one another during this motion,   
which concludes the proof. 
\ifijrrv
\qed
\else
\qedhere 
\fi

\end{proof}

\fi
We emphasize that even though our proof finds a trajectory which uses edges where a single robot moves at a time, the solution that would be found in practice is not necessarily restricted to individual-robot moves. This is due to the fact the tensor roadmap also includes edges representing simultaneous motion of several robots.

\subsection{Discussion}
Theorem~\ref{thm:multi_robot} implies that if a given MRMP planner is guaranteed to find an optimal collision-free path over a TR, then it is also guaranteed to find a $(1+\eps)$-approximation of the optimal $\vec{\delta}$-clear trajectory (in the continuous domain), when each \prm graph is constructed using the sample set $\X_{\omega\delta_i,\delta_i}$ and radius $r_i={{\delta_i(\eps+1)}/{(\eps+2)}}$ for each robot $1\leq i\leq R$, where $\omega={{\eps}/(2(\eps+2))}$.

This statement applies, for instance, to \mstar and \mccbs. The former can be viewed as a refined version of~\astar for searching the TR. The latter implicitly explores the TR by incrementally considering combinations of single-robot trajectories induced by the \prm graphs, until a combination that yields a collision-free composite trajectory is found. The \drrtstar planner implicitly explores the TR via an \rrt-style random exploration using a secondary sampling procedure which is employed after the \prm graphs are constructed. Due to this additional randomization step, \drrtstar achieves a $(1+\eps)$-approximation only  \emph{asymptotically}. Nevertheless, our analysis simplifies the usage of this algorithm by derandomizing the construction of \prm graphs used in \drrtstar.

Finally, we provide an example for the number of samples that should be used according to Theorem~\ref{thm:multi_robot} within each \prm roadmap for specific parameters. In particular, we report in
\ifextendedv
Table~\ref{tbl:sampcomp_examples}
\fi
the value $|\X_{\omega \delta_i,\delta_i}|$, for varying values of the stretch parameter $\eps$, dimension $d$, and clearance vector $(\delta_1,\ldots,\delta_R)$, where $\delta_i=0.1$ for all robots. 
While these values are quite large, particularly for higher dimensions, we emphasize that we do not expect our bounds to be tight, as observed in Section~\ref{sec:experiments}. This suggests that smaller sample sets are sufficient for $(\eps, \vec{\delta})$-completeness, which we aim to explore in future research (see Section~\ref{sec:conclusion}).

\ifextendedv
\def\tablecap{\caption{Multi-robot sample complexity example for $\delta=0.1$ with varying values for the dimension of each robot, $d$, and the required maximal stretch factor, $\eps$.}\label{tbl:sampcomp_examples}}
\fi
\newcommand{\cdotsmall}{\!\cdot\!}
\begin{table}[H]
\scriptsize 
\centerline{
\begin{tabular}{|c||c|c|c|c|c|}
\hline
 $d$ & $\eps=\infty$      & $\eps=5$       & $\eps=1$        & $\eps=0.5$      & $\eps=0.25$     \\
\hline \hline
$2$ & $181$      & $313$     & $1201$     & $3281$     & $1.05\cdotsmall10^4$    \\
\hline
$3$ & $2331$     & $6119$    & $5.68\cdotsmall10^4$    & $2.43\cdotsmall10^5$   & $1.43\cdotsmall10^6$  \\
\hline
$4$ & $4.93\cdotsmall10^4$    & $1.49\cdotsmall10^5$  & $2.83\cdotsmall10^6$  & $2.19\cdotsmall10^7$ & $2.21\cdotsmall 10^8$ \\
\hline
$5$ & $9.09\cdotsmall10^5$   & $4.37\cdotsmall10^6$ & $1.69\cdotsmall 10^8$ & $2.23\cdotsmall 10^9$ & $3.94\cdotsmall 10^{10}$ \\
\hline
$6$ & $1.89\cdotsmall10^7$ & $1.5\cdotsmall 10^8$ & $1.18\cdotsmall 10^{10}$ & $2.46\cdotsmall 10^{11}$ & $7.82\cdotsmall 10^{12}$\\
\hline
\end{tabular}
}
\tablecap
\end{table}

\section{Experimental results}\label{sec:experiments}
\newcommand{\sfwidth}{0.5\columnwidth}
\ifextendedv
\newcommand{\siwidth}{0.99\linewidth}
\fi
\begin{figure*}
    \centering
    \begin{subfigure}[t]{\sfwidth}
        \centering
        \includegraphics[trim=0 5 440 10, clip, width=\siwidth]{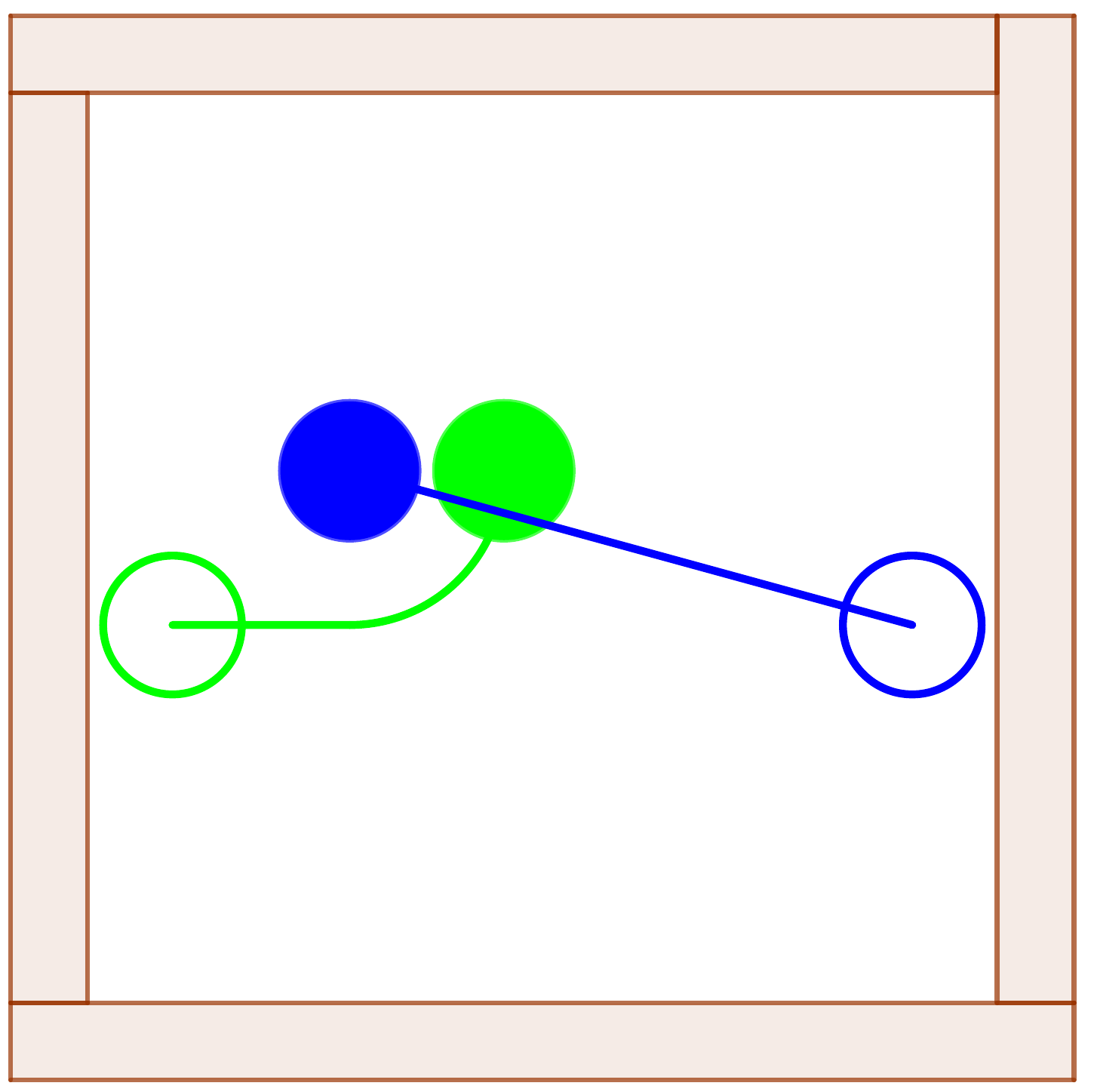}
    \end{subfigure}
    \begin{subfigure}[t]{\sfwidth}
        \centering
        \includegraphics[trim=0 20 470 10, clip, width=\siwidth]{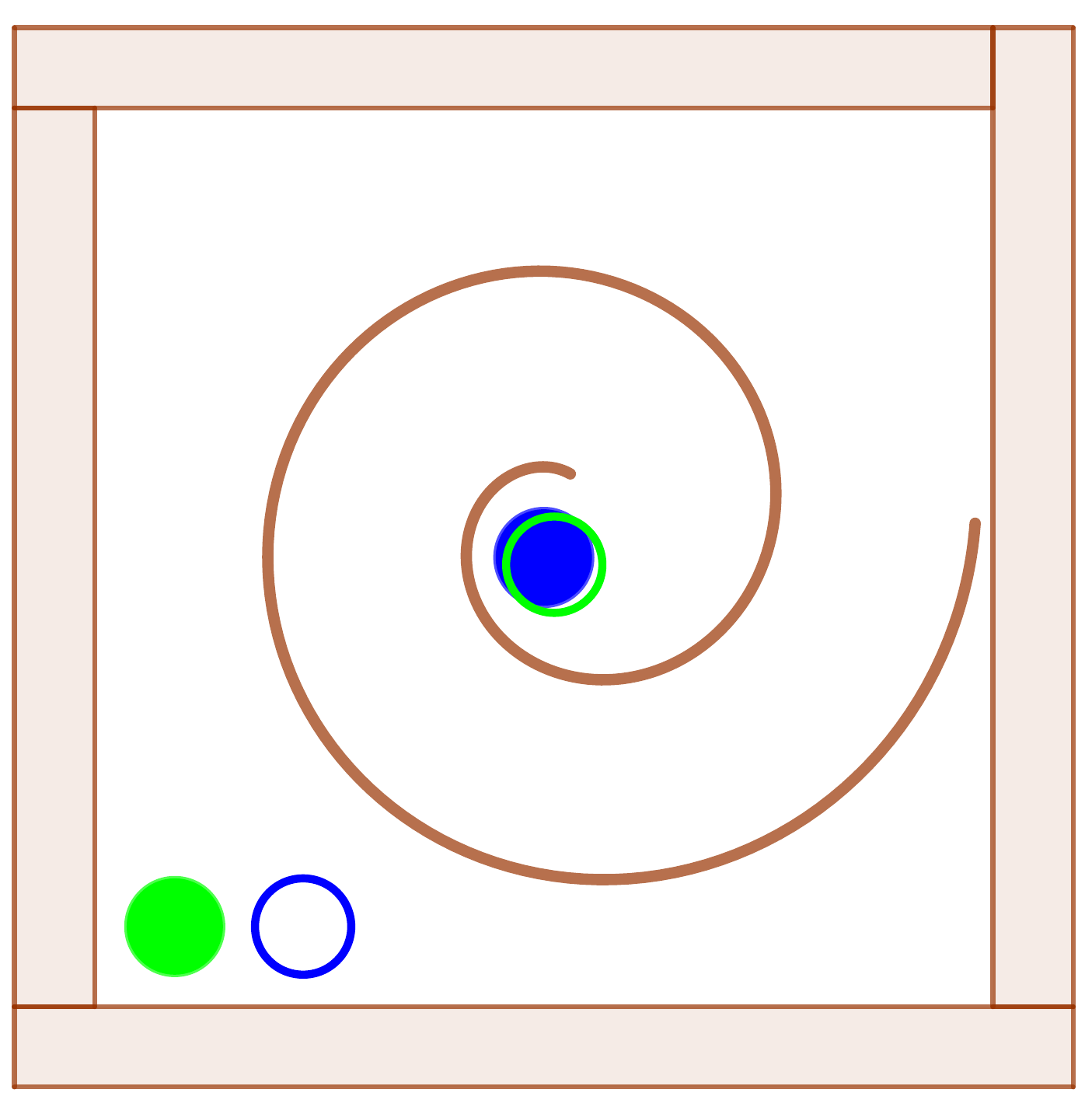}
    \end{subfigure}
    \begin{subfigure}[t]{\sfwidth}
        \centering
        \includegraphics[trim=45 5 400 10, clip,width=\siwidth]{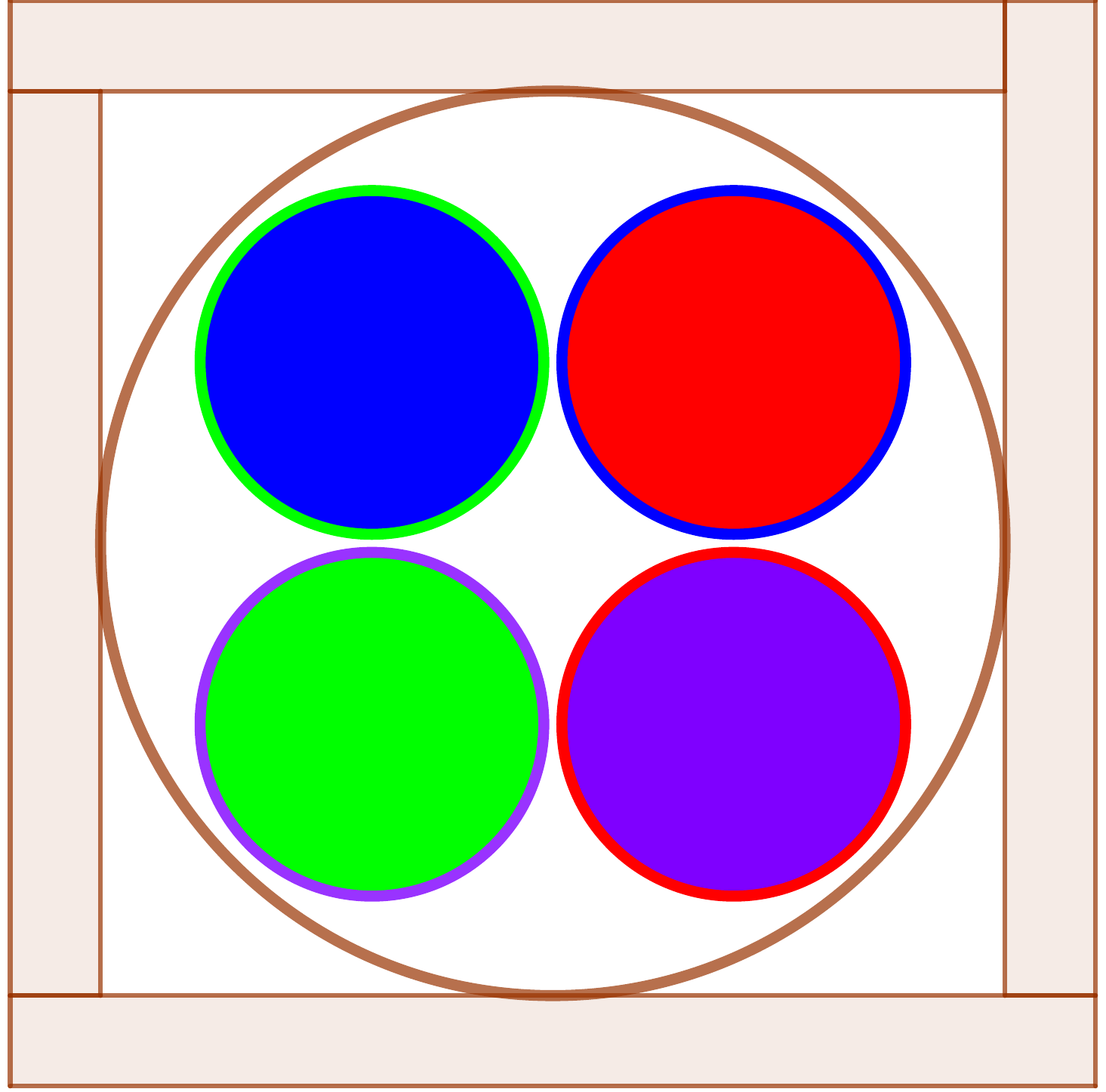}
    \end{subfigure}
    \begin{subfigure}[t]{\sfwidth}
        \centering
        \includegraphics[trim=0 5 420 0, clip, width=\siwidth]{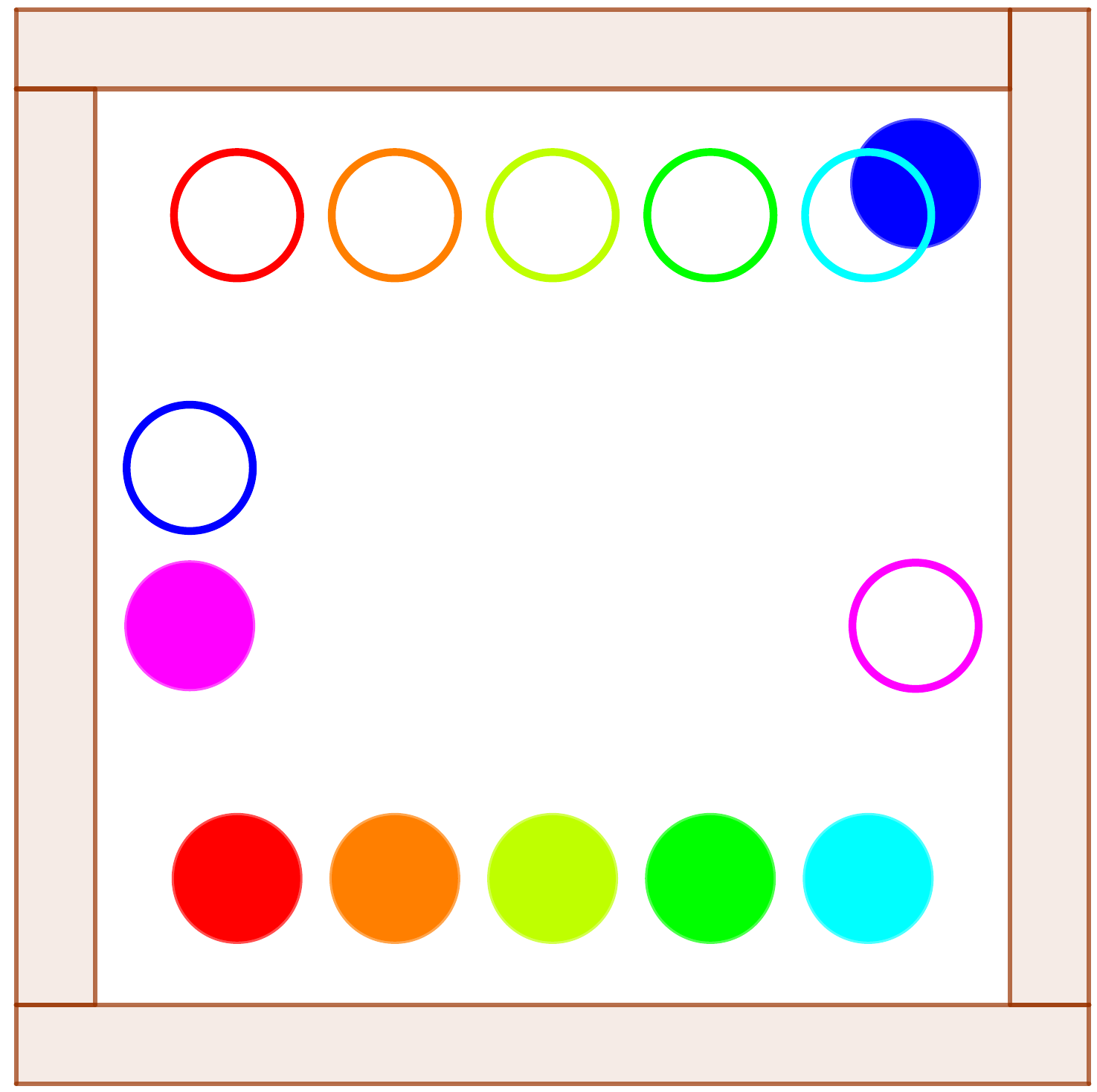}
    \end{subfigure}
    \caption{\sf{Test scenarios for multiple disc robots, where a circle and a disc of the same color represent the start and goal positions, respectively, of a robot. [Left] A $2$-robot obstacle-free scenario. The optimal trajectories, which were derived in~\directcite{kirkpatrick2016characterizing}, are drawn for each robot. The robots' radius is $0.09$ and the static clearance $\mu$ is equal to $0.02$. The latter value measures 
    the minimum over (i) the inter-robot distances at the initial placement, (ii) the distance of each robot at the initial placement from the obstacles, and (iii,iv) the respective quantities for the target placement. [Center-left] A $2$-robot scenario with a spiral obstacle with robots' radius $0.06$ and  static clearance $\mu=0.04$. {[Center-right] Four robots tightly placed within a circular barrier,  using robot radius $0.19$ and static clearance $\mu=0.02$.} [Right] A $7$-robot scenario with robots' radius  $0.08$ and the static clearance  $\mu=0.04$. \label{fig:scenarios}}}
    \ifextendedv
    \vspace{-5pt}
    \fi
\end{figure*}
We provide experimental results to support our theoretical findings, focusing on the case of multiple disc robots operating in a planar domain. We study the effect that the stretch parameter $\eps$, which determines the structure of the underlying \prm graphs within the tensor product graph $\widehat{G}$ in Theorem~\ref{thm:multi_robot}, has on the actual solution quality. We observe that the resulting approximation factor is in fact significantly lower than $1+\eps$. This suggests that our analysis can be further refined to support even sampling distributions with fewer samples (see Section~\ref{sec:conclusion}).

Concerning the design of the experiments, we were faced with two challenges. 
First, current MRMP algorithms~\indirectcite{LiEA2019,WagCho15,DBLP:journals/arobots/ShomeSDHB20} (including our own) are limited in their ability to cope with the large tensor roadmaps required to guarantee path quality according to our analysis; this curbed our ability to go to large numbers of robots in the experiments, and strongly motivates further improving of such algorithms. Secondly, in order to provide meaningful experimental reports, we need yardsticks to compare to; the problem is that optimal MRMP algorithms are not known to be tractable even for the simple case of two unit disc robots moving amid obstacles in the plane, and it is highly non-trivial to calculate optimal solutions when coordination is required.  We explain below how we overcame this latter impediment. 

\subsection{Scenarios}
The scenarios are illustrated in Figure~\ref{fig:scenarios}. The first scenario, [Left], consists of two robots in an obstacle-free environment. We use this simple example to benchmark our solution against an optimal $\delta$-clear solution using a recent work that provides characterization of optimal trajectories for two disc robots \emph{in the absence of obstacles}~\indirectcite{kirkpatrick2016characterizing}. The [Center-left] scenario uses a more complicated workspace topology, which potentially requires more samples to achieve a near-optimal solution, since multiple straight-line segments are required to approximate every single-robot trajectory. {The third scenario, [Center-right], which consists of four robots, aims to test our theory for a tight setting, which is obtained by tightly packing the four robots inside a circular barrier.} The fourth scenario, [Right], which consists of seven robots, aims to test our theory for a larger number of robots, where additional coordination is required to achieve a solution. 

\subsection{Results}
To test our theory on the aforementioned scenarios, we constructed \prm graphs using our staggered grid as the sample set, corresponding to different values of the stretch parameter $\eps$ and clearance $\delta$ (for simplicity, we use the same value of $\delta$ across all robots, i.e., $\vec{\delta}:=\{\delta,\ldots,\delta\}$). 
For all the scenarios, we set the value $\delta$ to be equal to the \emph{static clearance} $\mu$ (see caption of Figure~\ref{fig:scenarios}). 
We set the stretch parameter $\eps$ to different values in the range $[0.75,\infty)$. We then use \astar-search for the first three scenarios, and \mccbs for the seven-robot scenario, to obtain the best solution from the resulting tensor roadmap. 

\ifextendedv
\newcommand{\expreswidth}{\linewidth}
\fi

\begin{figure}
  \centering
    \includegraphics[trim=0 10 10 25, clip, width=\expreswidth]{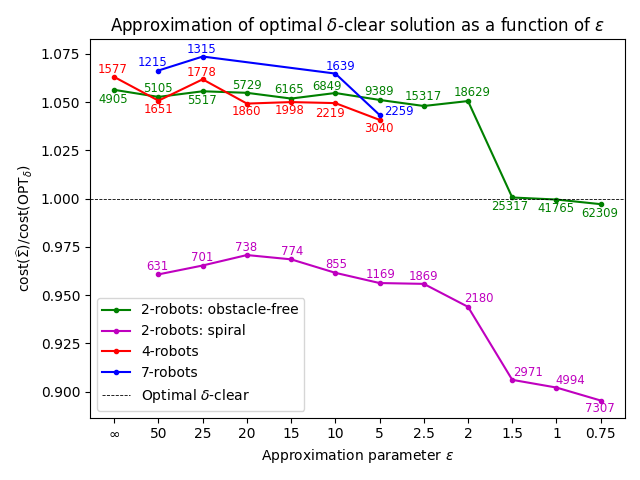}
  \caption{\sf{We report for each of the four scenarios its approximation factors, which are represented by the ratio between the cost of the solution obtained from the tensor roadmap and $\textup{OPT}_\delta$. Notice that the tensor-roadmap solution is not necessarily $\delta$-clear and thus can be of lower cost than $\textup{OPT}_\delta$, which explains why two of the plots get approximation factors smaller than $1$. Next to the points in each plot is the number of collision-free samples in the \prm graph of an individual robot. Due to the prohibitive running times incurred by \mccbs on the $7$-robot scenario {and by our Python implementations of \astar on the $4$-robot scenario}, we report the solution quality for these scenarios only for a subset of  values of $\eps$.
  \label{fig:exp_res}}}
\end{figure}

The results are reported in Figure~\ref{fig:exp_res}, where we plot the approximation ratio obtained using the staggered grid $\X_{\omega\delta,\delta}$ set to guarantee an approximation factor of at most $1+\eps$ (see Theorem~\ref{thm:multi_robot}). The reported approximation factor represents the ratio between the cost of the best solution obtained from the tensor roadmap, denoted by $\widehat{\Sigma}$, 
and the optimal $\delta$-clear solution whose cost is denoted by $\textup{OPT}_\delta$.  For the first scenario we obtain the value $\textup{OPT}_\delta$ using~\directcite{kirkpatrick2016characterizing}, as we reported earlier. For the [Center-left] scenario, $\textup{OPT}_\delta$ is equal to 
the sum of the shortest $\delta$-clear trajectories for the two individual robots, as one of the robots can move after the other finishes its motion without increasing the overall cost of the solution.
{For the [Center-right] scenario, $\textup{OPT}_\delta$ is equal to the perimeter of the circle going though the robot centers at the initial positions (as each robot traverses a quarter of the circle).}
For the [Right] scenario, $\textup{OPT}_\delta$ is equal to the sum of the Euclidean distances between each robot's origin and destination, as  the robots can move one after the other to obtain an optimal solution.

In all the experiments we obtain an approximation factor that is significantly lower than what our worst-case analysis predicts. For instance, already when setting the stretch parameter to $\eps=50$ we obtain an approximation factor of at most $1.075$. Moreover, for $\eps\leq 1.5$ we obtain approximation factors below $1$ in both two-robot scenarios. This is possible as the solution obtained from the tensor roadmap is not necessarily $\delta$-clear, which allows robots to take shortcuts in proximity to obstacles and each other. 
The overall trend of the graphs complies with our expectation: when the stretch parameter decreases, we obtain improved solutions. The biggest improvement (at least for the two-robot scenarios) 
occurs when $\eps$ goes below $2$, due to the rapid increase in the number of samples in $\X_{\omega\delta,\delta}$. In some cases using a smaller number of samples may yield better solutions, as in the 2-robot spiral scenario for stretch factors $\eps=50$ and $\eps=20$. The explanation is that the smaller sample set gets closer to the (approximate) optimal solution by chance. Still, the \emph{worst-case approximation factor} is guaranteed to improve as the size of the staggered grid increases.

\iffuturev

\subsection{Comparing the staggered grid with random sampling}\label{sec:random_vs_sg}
The staggered grid offers good theoretical guarantees on the necessary sample size for obtaining near-optimal solution in the single robot case and consequently in the multi-robot case as well.  
An important question is whether the staggered grid also provides an improvement over the somewhat standard random uniform sampling in practice. In this section we demonstrate that the staggered grid provides a practical improvement over random uniform sampling especially for small sample sets.

Next, we show how the staggered grid outperforms random sampling in both of our test scenes with $2$ robots, those depicted in the [Left] and [Center-Left] sub-figures of Figure~\ref{fig:scenarios}. The testing scheme we use is the following: for varying values of $\eps$, we calculate the number of points in the staggered grid, and the connection radius prescribed by Theorem~\ref{thm:multi_robot}. We then randomly choose that same number of points uniformly and use the same connection radius for solving the scene. This process is repeated $10$ times to account for randomness. 

We first look at the success rate. Using the staggered grid we are guaranteed to have a perfect success rate as long as there exists a path with  $\delta$-clearance, for a given $\delta$. For a random set of points this is no longer the case---see the success rate reported in Figure~\ref{fig:rnd_vs_sg_success_rate}. Note that as $\eps$ decreases (and the number of samples increases) the success rate improves. We also observe that the scene with obstacles has a very low success rate for small values of $\eps$. 

\begin{figure}[H]
  \centering
    \includegraphics[trim=0 12 10 25, clip, width=\expreswidth]{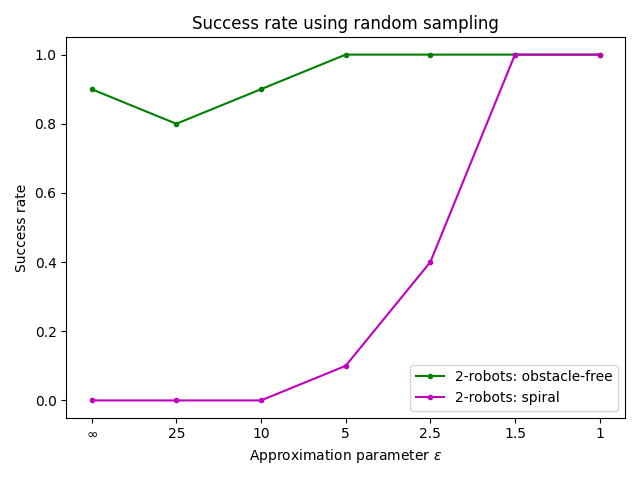}
  \caption{\sf{The success rate in finding a valid solution for the MRMP problem, using a sample set of random points (uniformly distributed in $[\delta,1-\delta]^2$). The success rate is measured with respect to $10$ runs. Note that the staggered grid achieves a success rate of 100\%. 
  \label{fig:rnd_vs_sg_success_rate}}}
\end{figure}

Next we consider the cost of the resulting trajectory. We compare the average cost of the path generated using the randomly sampled points with the cost of the path generated using the staggered grid. As can be observed in Figure~\ref{fig:rnd_vs_sg_cost}, for small sets of points (large $\eps$ values) the staggered grid outperforms an average uniformly chosen random set of sample points. This advantage decreases for smaller $\eps$ values. 

To summarize, our results emphasize the strength of the staggered grid for large values of~$\eps$, when compared with the standard random sampling both in terms of success rate and solution quality. 

To conclude this section, we briefly discuss how to choose $\eps$ and $\delta$ in practice. Considering that the number of samples rapidly increases as $\eps$ and $\delta$ decrease, we suggest starting with a desired value of $\delta$ but setting $\eps$ to be infinity. In our experience, and as is evident in the experimental results, even large values of $\eps$ often yield low stretch.  If time permits, one can decrease $\eps$ (and rerun the algorithm). In case that there is no desired $\delta$ that is prescribed from the application, it can also be set initially to $\infty$, and then gradually reduced. We leave the study of more sophisticated parameter tuning methods for future research.

\begin{figure}[H]

  \centering
  \begin{subfigure}[b]{0.49\linewidth}
  \centering
    \includegraphics[trim=10 17 10 10, clip, width=1\textwidth]{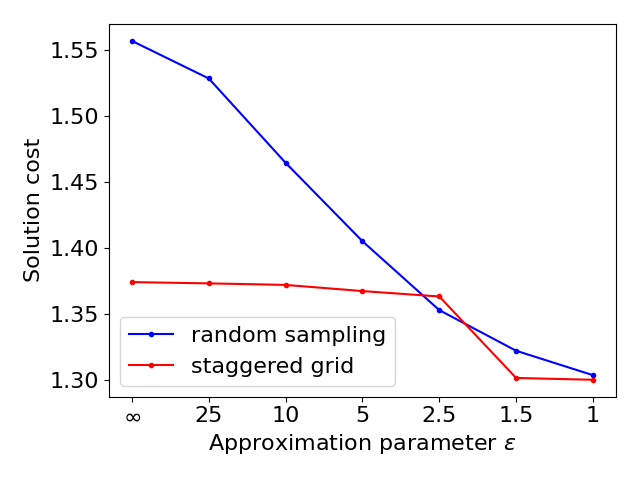}
    \caption{obstacle-free}
  \end{subfigure}
  \begin{subfigure}[b]{0.49\linewidth}
  \centering
    \includegraphics[trim=10 17 10 10, clip, width=1\textwidth]{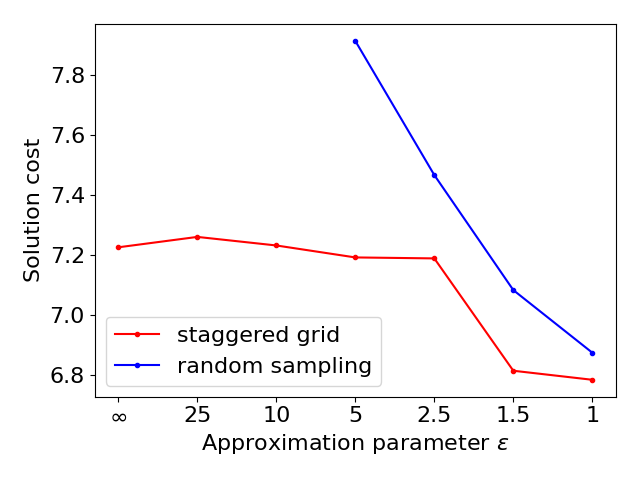}
    \caption{spiral}
  \end{subfigure}

  \caption{\sf{The average cost of successful runs for the MRMP problem, for the two scenarios of two robots each. We show the average cost using randomly sampled points (blue line) and the cost using the staggered grid (red line).
  \label{fig:rnd_vs_sg_cost}}}
\end{figure}

\fi

\section{Discussion and future work}\label{sec:conclusion}
We developed sufficient theoretical conditions for finite-sample near-optimality of the tensor roadmap, which is an underlying structure in several sampling-based algorithms for MRMP.
We also presented a new sampling scheme, termed the staggered grid, for near-optimal motion planing for individual robots, which requires fewer samples than previous work. 

Our work raises interesting questions for further investigation both in practice and theory.
The scalability issues we encountered when testing our theoretical finding on \mccbs---a continuous extension of a state-of-the-art method for MAPF---motivate the study of more effective methods for exploring tensor roadmaps introduced by large \prm graphs. On the positive side, the scenarios that we did manage to solve suggest that near-optimality can be achieved with smaller sample sets than our theory prescribes. This motivates the development of even more compact sampling-distributions for the single-robot case, and refining our proof technique for the multi-robot case (Theorem~\ref{thm:multi_robot}). In this context, we mention that there exists a sample distribution by~\cite{coxeter1959covering} that requires fewer samples than the staggered grid, albeit it only applies currently for $d=2$, and it is unclear whether it can be extended to higher dimensions $d\geq 3$. 

Additional practical enhancements to our approach could be the incorporation of learning-based sample distributions, which could  soften the blow of the curse of dimensionality~\cite{DBLP:conf/icra/IchterHP18}. It would also be interesting to consider time-based distance functions (e.g., minimize total travel time of robots) by considering \prm graphs that cover a time-parametrized configuration space.

{\small 
\section*{Acknowledgments}
The authors thank Matt Tsao for fruitful discussions, Nir Goren for providing a Python motion-planning framework implementation, and Jiaoyang Li for providing an implementation for \mccbs. }

{\small 
\section*{Funding}
Work by D.\ Dayan and D.\ Halperin was supported in part by the Israel Science Foundation (grant no.~1736/19), by the US NSF/US-Israel BSF (grant no. 2019754), by the Isarel Ministry of Science and Technology (grant no.~103129), by the Blavatnik Computer Science Research Fund, and by the Yandex Machine Learning Initiative for Machine Learning at Tel Aviv University. 

\noindent
Work by K.\ Solovey and M.\ Pavone  was supported in part by the Toyota Research Institute (TRI) and the Center for Automotive Research at Stanford (CARS).
}

\ifijrrv
\bibliographystyle{SageH}
\else
\bibliographystyle{IEEEtran}
\fi
\bibliography{enmop}


\ifextendstaggered
\newpage
\ifextendstaggered

\section{Staggered grid size analysis}\label{sec:staggered_grid_analysis}
In this section we further analyze the size of the staggered grid for completeness. We separate it to a unique section as the staggered grid is our solution to the general problem of \xcover{$\beta$} of a $d$-dimensional cube. Our benchmark will be the previous results from~\cite[Theorem~3]{TsaoSoloveyETAL2020}.

We now denote the relevant sizes. Let $\gamma, \beta > 0$ be the parameters as described for Lemma~\ref{lem:beta_cover}, and $d \geq 2$ the dimension of the $\gamma$-cube. We denote by $\Y_{\textup{lower}}$ the lower bound for the cube cover problem. Namely, \[\Y_{\textup{lower}} \approx \sqrt{\pi d} \bigpar{ \sqrt{\frac{d}{2 \pi e}} \cdot \frac{ 1-2\gamma}{\beta}}^d\;.\]
We denote by $\Y_{\textup{prev}}$ the previous upper bound for the cube cover problem. Namely, \[\Y_{\textup{prev}} \approx \sqrt{\pi d} \bigpar{ \sqrt{\frac{2d}{ \pi e}} \cdot \frac{ 1-2\gamma+\beta}{\beta}}^d\;.\]
Lastly, we denote by $\Y_{\textup{staggered}}$ the size of the staggered grid as defined in \ref{def:staggered_grid}. Namely, \[\Y_{\textup{staggered}} = \left(\left\lceil\frac{(1-2\gamma)\sqrt{d}}{\sqrt{8}\beta }\right\rceil\right)^d+\left(\left\lceil\frac{(1-2\gamma)\sqrt{d}}{\sqrt{8}\beta }\right\rceil+1\right)^d\;.\]

To analyze our results we use both the asymptotic regime, i.e. $\beta$ and $\gamma$ tend to $0$, and give a detailed table with exact results for selected values.

For the asymptotic regime, we first compare the quantities $\Y_{\textup{staggered}}$ and $\Y_{\textup{lower}}$ for small values of $\beta$ and $\gamma$. We consider the ratio $\frac{\Y_{\textup{staggered}}}{\Y_{\textup{lower}}}$ 
in the asymptotic regime where the values 
$\beta$ and $\gamma$ tend to zero.

\begin{equation}
\begin{aligned}
\label{eq:cube_cover_ratio}
\frac{\Y_{\textup{staggered}}}{\Y_{\textup{lower}}} &=\frac{\left(\left\lceil\frac{(1-2\gamma)\sqrt{d}}{\sqrt{8}\beta }\right\rceil\right)^d+\left(\left\lceil\frac{(1-2\gamma)\sqrt{d}}{\sqrt{8}\beta }\right\rceil+\bm{1}\right)^d}{\sqrt{\pi d} \bigpar{ \sqrt{\frac{d}{2 \pi e}} \cdot \frac{ 1-2\gamma}{\beta}}^d} \\
&\overset{(a)}{\approx} \frac{2\left(\frac{(1-2\gamma)\sqrt{d}}{\sqrt{8}\beta }\right)^d}{\sqrt{\pi d} \bigpar{ \sqrt{\frac{d}{2 \pi e}} \cdot \frac{ 1-2\gamma}{\beta}}^d}\\
&=\frac{2}{\sqrt{\pi d}} \bigpar{\sqrt{\frac{\pi e}{4}}}^d \approx \frac{2}{\sqrt{\pi d}}1.4611
^d\;,
\end{aligned}
\end{equation}
where the approximation in $(a)$ is due to ignoring the highlighted unit value $1$ and the roundups (the influence of both values tends to $0$ as $\beta$ or $\gamma$ tend to $0$). 
This implies that our result yields cover sets that are bigger than the lower bound by an exponential factor in $d$, but this factor is only $1.4611$ compared to an exponential factor of more than $2$ in the previous work. \dror{TODO also note this is not the motion-planning lower bound, the motion-planning bound is worse.}

Still in the asymptotic regime, we also compare the quantities $\Y_{\textup{staggered}}$ and $\Y_{\textup{prev}}$ for small values of $\beta$ and $\gamma$. We consider the ratio $\frac{\Y_{\textup{prev}}}{\Y_{\textup{staggered}}}$, again for the case where the values $\beta$ and $\gamma$ tend to zero.

\begin{equation}
\begin{aligned}
\label{eq:cube_cover_ratio_upper}
\frac{\Y_{\textup{prev}}}{\Y_{\textup{staggered}}} &=\frac{\sqrt{\pi d} \bigpar{ \sqrt{\frac{2d}{ \pi e}} \cdot \frac{ 1-2\gamma+\beta}{\beta}}^d}{\left(\left\lceil\frac{(1-2\gamma)\sqrt{d}}{\sqrt{8}\beta }\right\rceil\right)^d+\left(\left\lceil\frac{(1-2\gamma)\sqrt{d}}{\sqrt{8}\beta }\right\rceil+1\right)^d} \\ & \approx \frac{\sqrt{\pi d}}{2}1.3687^d\;,
\end{aligned}
\end{equation}
we omit the calculations as this are similar to those of Eq~\ref{eq:ratio}. Note that the cube cover improvement ratio is exactly the same as the motion-planning improvement ratio.

In Table~\ref{tbl:staggered_grid_size} we compare our results with the previous lower and upper bounds. The table support the theoretical results from Eq~\ref{eq:cube_cover_ratio} and Eq~\ref{eq:cube_cover_ratio_upper}. Furthermore it shows that for most feasible case ($d\leq 10$) the staggered grid is closer to the lower bound than to the previous upper bound. \dror{we can even say staggered/lower $\leq$ prev/staggered.} This table can be compared with Table~\ref{tbl:samp_comp_comparison} to show the difference between the lower bound to staggered grid in the 2 cases, note that for motion-planning with $\eps=\infty$ (the only one with a lower bound), the relevant cover results are those where $\beta=\gamma$. This leads us to the conclusion that while there are probably better ways to produce a \xcover{$\beta$} these will not be enough to close the gap between the motion-planning lower bound and our results, that is lowering the upper bound for the cube cover problem does not mean tight bound for the motion-planning problem. To close the motion-planning gap one would probably need to also increase the lower bound (from \cite[Theorem~1]{TsaoSoloveyETAL2020}) or find better ways to use the \xcover{$\beta$} (compared to \cite[Theorem~2]{TsaoSoloveyETAL2020} and Lemma~\ref{lem:super_duper}).

\begin{table*}[b]
\centerline{

\begin{tabular}{|c|c||c|c|c||c|c|c||c|c|c|}
\cline{3-11}
                      \multicolumn{2}{c|}{}   & \multicolumn{3}{c||}{$\beta=0.1$}       & \multicolumn{3}{c||}{$\beta=0.05$} & \multicolumn{3}{c|}{$\beta=0.01$} \\ \hline
$\gamma$  &  $d$                                   & $\Y_{\textup{lower}}$       & $\Y_{\textup{staggered}}$      & $\Y_{\textup{prev}}$     & $\Y_{\textup{lower}}$       & $\Y_{\textup{staggered}}$      & $\Y_{\textup{prev}}$     & $\Y_{\textup{lower}}$       & $\Y_{\textup{staggered}}$      & $\Y_{\textup{prev}}$  \\ \hline \hline
\multirow{9}{*}{$0.05$}                  & $2 $  & $26      $ & $61      $ & $128     $ & $104     $ & $181     $ & $460     $ & $2579    $ & $4141    $ & $1.05\cdot10^{4}$ \\ \cline{2-11}
                  & $3 $  & $175     $ & $559     $ & $1910    $ & $1393    $ & $3925    $ & $1.31\cdot10^{4}$ & $1.74\cdot10^{5}$ & $3.61\cdot10^{5}$ & $1.44\cdot10^{6}$ \\ \cline{2-11}
                  & $4 $  & $1330    $ & $6497    $ & $3.24\cdot10^{4}$ & $2.13\cdot10^{4}$ & $6.70\cdot10^{4}$ & $4.23\cdot10^{5}$ & $1.33\cdot10^{7}$ & $3.46\cdot10^{7}$ & $2.22\cdot10^{8}$ \\ \cline{2-11}
                  & $5 $  & $1.12\cdot10^{4}$ & $9.18\cdot10^{4}$ & $6.08\cdot10^{5}$ & $3.59\cdot10^{5}$ & $1.81\cdot10^{6}$ & $1.51\cdot10^{7}$ & $1.12\cdot10^{9}$ & $4.01\cdot10^{9}$ & $3.79\cdot10^{10}$ \\ \cline{2-11}
                  & $6 $  & $1.03\cdot10^{5}$ & $7.94\cdot10^{5}$ & $1.24\cdot10^{7}$ & $6.58\cdot10^{6}$ & $4.09\cdot10^{7}$ & $5.83\cdot10^{8}$ & $1.03\cdot10^{11}$ & $4.68\cdot10^{11}$ & $7.03\cdot10^{12}$ \\ \cline{2-11}
                  & $7 $  & $1.01\cdot10^{6}$ & $1.48\cdot10^{7}$ & $2.71\cdot10^{8}$ & $1.30\cdot10^{8}$ & $1.02\cdot10^{9}$ & $2.42\cdot10^{10}$ & $1.01\cdot10^{13}$ & $6.69\cdot10^{13}$ & $1.40\cdot10^{15}$ \\ \cline{2-11}
                  & $8 $  & $1.06\cdot10^{7}$ & $1.43\cdot10^{8}$ & $6.31\cdot10^{9}$ & $2.72\cdot10^{9}$ & $2.80\cdot10^{10}$ & $1.07\cdot10^{12}$ & $1.06\cdot10^{15}$ & $9.01\cdot10^{15}$ & $2.97\cdot10^{17}$ \\ \cline{2-11}
                  & $9 $  & $1.17\cdot10^{8}$ & $3.36\cdot10^{9}$ & $1.55\cdot10^{11}$ & $6.01\cdot10^{10}$ & $1.31\cdot10^{12}$ & $5.01\cdot10^{13}$ & $1.17\cdot10^{17}$ & $1.45\cdot10^{18}$ & $6.64\cdot10^{19}$ \\ \cline{2-11}
                  & $10$  & $1.37\cdot10^{9}$ & $8.79\cdot10^{10}$ & $4.02\cdot10^{12}$ & $1.40\cdot10^{12}$ & $4.32\cdot10^{13}$ & $2.46\cdot10^{15}$ & $1.37\cdot10^{19}$ & $2.32\cdot10^{20}$ & $1.56\cdot10^{22}$ \\ \hline \hline
\multirow{9}{*}{$0.01$}                  & $2 $  & $31      $ & $61      $ & $149     $ & $123     $ & $221     $ & $541     $ & $3058    $ & $4901    $ & $1.25\cdot10^{4}$ \\ \cline{2-11}
                  & $3 $  & $225     $ & $855     $ & $2406    $ & $1798    $ & $4941    $ & $1.67\cdot10^{4}$ & $2.25\cdot10^{5}$ & $4.65\cdot10^{5}$ & $1.85\cdot10^{6}$ \\ \cline{2-11}
                  & $4 $  & $1870    $ & $6497    $ & $4.41\cdot10^{4}$ & $2.99\cdot10^{4}$ & $8.90\cdot10^{4}$ & $5.84\cdot10^{5}$ & $1.87\cdot10^{7}$ & $4.94\cdot10^{7}$ & $3.11\cdot10^{8}$ \\ \cline{2-11}
                  & $5 $  & $1.72\cdot10^{4}$ & $9.18\cdot10^{4}$ & $8.93\cdot10^{5}$ & $5.50\cdot10^{5}$ & $2.47\cdot10^{6}$ & $2.26\cdot10^{7}$ & $1.72\cdot10^{9}$ & $5.96\cdot10^{9}$ & $5.78\cdot10^{10}$ \\ \cline{2-11}
                  & $6 $  & $1.71\cdot10^{5}$ & $1.53\cdot10^{6}$ & $1.97\cdot10^{7}$ & $1.10\cdot10^{7}$ & $5.81\cdot10^{7}$ & $9.46\cdot10^{8}$ & $1.71\cdot10^{11}$ & $7.82\cdot10^{11}$ & $1.17\cdot10^{13}$ \\ \cline{2-11}
                  & $7 $  & $1.84\cdot10^{6}$ & $2.95\cdot10^{7}$ & $4.64\cdot10^{8}$ & $2.35\cdot10^{8}$ & $2.17\cdot10^{9}$ & $4.26\cdot10^{10}$ & $1.84\cdot10^{13}$ & $1.16\cdot10^{14}$ & $2.53\cdot10^{15}$ \\ \cline{2-11}
                  & $8 $  & $2.10\cdot10^{7}$ & $3.14\cdot10^{8}$ & $1.17\cdot10^{10}$ & $5.37\cdot10^{9}$ & $6.34\cdot10^{10}$ & $2.05\cdot10^{12}$ & $2.10\cdot10^{15}$ & $1.77\cdot10^{16}$ & $5.82\cdot10^{17}$ \\ \cline{2-11}
                  & $9 $  & $2.53\cdot10^{8}$ & $7.52\cdot10^{9}$ & $3.10\cdot10^{11}$ & $1.29\cdot10^{11}$ & $2.00\cdot10^{12}$ & $1.04\cdot10^{14}$ & $2.53\cdot10^{17}$ & $2.97\cdot10^{18}$ & $1.42\cdot10^{20}$ \\ \cline{2-11}
                  & $10$  & $3.20\cdot10^{9}$ & $8.79\cdot10^{10}$ & $8.67\cdot10^{12}$ & $3.28\cdot10^{12}$ & $6.80\cdot10^{13}$ & $5.53\cdot10^{15}$ & $3.20\cdot10^{19}$ & $5.43\cdot10^{20}$ & $3.63\cdot10^{22}$ \\ \hline \hline
\multirow{9}{*}{$0$}                  & $2 $  & $32      $ & $61      $ & $155     $ & $128     $ & $221     $ & $562     $ & $3184    $ & $5101    $ & $1.30\cdot10^{4}$ \\ \cline{2-11}
                     & $3 $  & $239     $ & $855     $ & $2543    $ & $1910    $ & $4941    $ & $1.77\cdot10^{4}$ & $2.39\cdot10^{5}$ & $4.88\cdot10^{5}$ & $1.97\cdot10^{6}$ \\ \cline{2-11}
                     & $4 $  & $2027    $ & $1.07\cdot10^{4}$ & $4.75\cdot10^{4}$ & $3.24\cdot10^{4}$ & $1.16\cdot10^{5}$ & $6.31\cdot10^{5}$ & $2.03\cdot10^{7}$ & $5.23\cdot10^{7}$ & $3.37\cdot10^{8}$ \\ \cline{2-11}
                     & $5 $  & $1.90\cdot10^{4}$ & $9.18\cdot10^{4}$ & $9.79\cdot10^{5}$ & $6.08\cdot10^{5}$ & $2.47\cdot10^{6}$ & $2.48\cdot10^{7}$ & $1.90\cdot10^{9}$ & $6.76\cdot10^{9}$ & $6.39\cdot10^{10}$ \\ \cline{2-11}
                     & $6 $  & $1.94\cdot10^{5}$ & $1.53\cdot10^{6}$ & $2.19\cdot10^{7}$ & $1.24\cdot10^{7}$ & $8.11\cdot10^{7}$ & $1.06\cdot10^{9}$ & $1.94\cdot10^{11}$ & $8.98\cdot10^{11}$ & $1.31\cdot10^{13}$ \\ \cline{2-11}
                     & $7 $  & $2.12\cdot10^{6}$ & $2.95\cdot10^{7}$ & $5.28\cdot10^{8}$ & $2.71\cdot10^{8}$ & $2.17\cdot10^{9}$ & $4.88\cdot10^{10}$ & $2.12\cdot10^{13}$ & $1.35\cdot10^{14}$ & $2.90\cdot10^{15}$ \\ \cline{2-11}
                     & $8 $  & $2.46\cdot10^{7}$ & $3.14\cdot10^{8}$ & $1.35\cdot10^{10}$ & $6.31\cdot10^{9}$ & $6.34\cdot10^{10}$ & $2.39\cdot10^{12}$ & $2.46\cdot10^{15}$ & $2.08\cdot10^{16}$ & $6.83\cdot10^{17}$ \\ \cline{2-11}
                     & $9 $  & $3.03\cdot10^{8}$ & $7.52\cdot10^{9}$ & $3.66\cdot10^{11}$ & $1.55\cdot10^{11}$ & $3.01\cdot10^{12}$ & $1.23\cdot10^{14}$ & $3.03\cdot10^{17}$ & $3.84\cdot10^{18}$ & $1.70\cdot10^{20}$ \\ \cline{2-11}
                     & $10$  & $3.92\cdot10^{9}$ & $2.00\cdot10^{11}$ & $1.04\cdot10^{13}$ & $4.02\cdot10^{12}$ & $1.05\cdot10^{14}$ & $6.70\cdot10^{15}$ & $3.92\cdot10^{19}$ & $6.50\cdot10^{20}$ & $4.44\cdot10^{22}$ \\ \hline
\end{tabular}
}

\caption{\sf{A comparison of the $\beta$-cover size for the $d$-dimensional $(1-2\gamma)$-cube with various values for the attributes of $\gamma$, ball radius $\beta$, and dimension $d$. For every combination of attributes we report the following values, $\Y_{\textup{lower}}$ is the lower bound for the cube cover problem, $\Y_{\textup{prev}}$ is the previous result, both are from~\cite[Theorem~3]{TsaoSoloveyETAL2020}. $\Y_{\textup{staggered}}$ is the number of points required to cover the cube using the staggered grid. Note that $\gamma=0$ is the cover of the unit cube, and that for the motion-planning problem $\beta \leq \gamma$.}}\label{tbl:staggered_grid_size}

\end{table*}

\fi
\fi

\appendix 

We provide an example of $\L,\L^o$ and $V_0,\ldots, V_\ell$ from Theorem~\ref{thm:multi_robot} for a two-robot setting with 
\begin{align*}
& T_1= (0, 0.1, 0.2, 0.3, 0.4, 0.5, 0.7, 0.9, 1), \\  
& Z_1=(z^1_0,z^1_1,z^1_2,z^1_3,z^1_4,z^1_5,z^1_6,z^1_7,z^1_8), \\
& T_2= (0, 0.2, 0.4, 0.6, 0.8, 0.85, 0.9, 0.95, 1), \\
& Z_2=(z^2_0,z^2_1,z^2_2,z^2_3,z^2_4,z^2_5,z^2_6,z^2_7,z^2_8).
\end{align*}
Note that $z^i_0$ is robot $i$'s starting point, and $z^i_8$ is its goal. In particular, 
\begin{align*}
\L=\{& (1,0.1,z^1_1),(1,0.2,z^1_2),(1,0.3,z^1_3),(1,0.4,z^1_4), \\
& (1,0.5,z^1_5),(1,0.7,z^1_6),(1,0.9,z^1_7),(1,1,z^1_8), \\
& (2,0.2,z^2_1),(2,0.4,z^2_2),(2,0.6,z^2_3),(2,0.8,z^2_4) \\
& (2,0.85,z^2_5),(2,0.9,z^2_6),(2,0.95,z^2_7),(2,1,z^2_8)\}.
\end{align*}
\begin{align*}
\L^o=\{& (1,0.1,z^1_1),(1,0.2,z^1_1),(2,0.2,z^2_1),(1,0.3,z^1_3), \\
& (1,0.4,z^1_4),(2,0.4,z^2_2),(1,0.5,z^1_5),(2,0.6,z^2_3), \\
& (1,0.7,z^1_6),(2,0.8,z^2_4),(2,0.85,z^2_5),(1,0.9,z^1_7), \\
& (2,0.9,z^2_6),(2,0.95,z^2_7),(1,1,z^1_8),(2,1,z^2_8)\},
\end{align*} 
yielding the vertices $V_0 = (z^1_0,z^2_0), V_1 = (z^1_1,z^2_0), V_2 = (z^1_2,z^2_0), V_3 = (z^1_2,z^2_1), V_4 = (z^1_3,z^2_1), V_5 = (z^1_4,z^2_1), V_6 = (z^1_4,z^2_2), V_7 = (z^1_5,z^2_2), V_8 = (z^1_5,z^2_3), V_9 = (z^1_6,z^2_3), V_{10} = (z^1_6,z^2_4), V_{11} = (z^1_6,z^2_5), V_{12} = (z^1_7,z^2_5), V_{13} = (z^1_7,z^2_6), V_{14} = (z^1_7,z^2_7), V_{15} = (z^1_8,z^2_7), V_{16} = (z^1_8,z^2_8).$

An illustration for two vertices can be found in Figure~\ref{fig:example-fig}. A full illustration for all vertices can be found at \href{https://www.geogebra.org/m/j3cz3fbb}{https://www.geogebra.org/m/j3cz3fbb}.

\begin{figure}[ht]
  \centering
  \begin{subfigure}[b]{0.95\linewidth}
  \centering
    \includegraphics[trim=5.5 1 8 1, clip, width=1\textwidth]{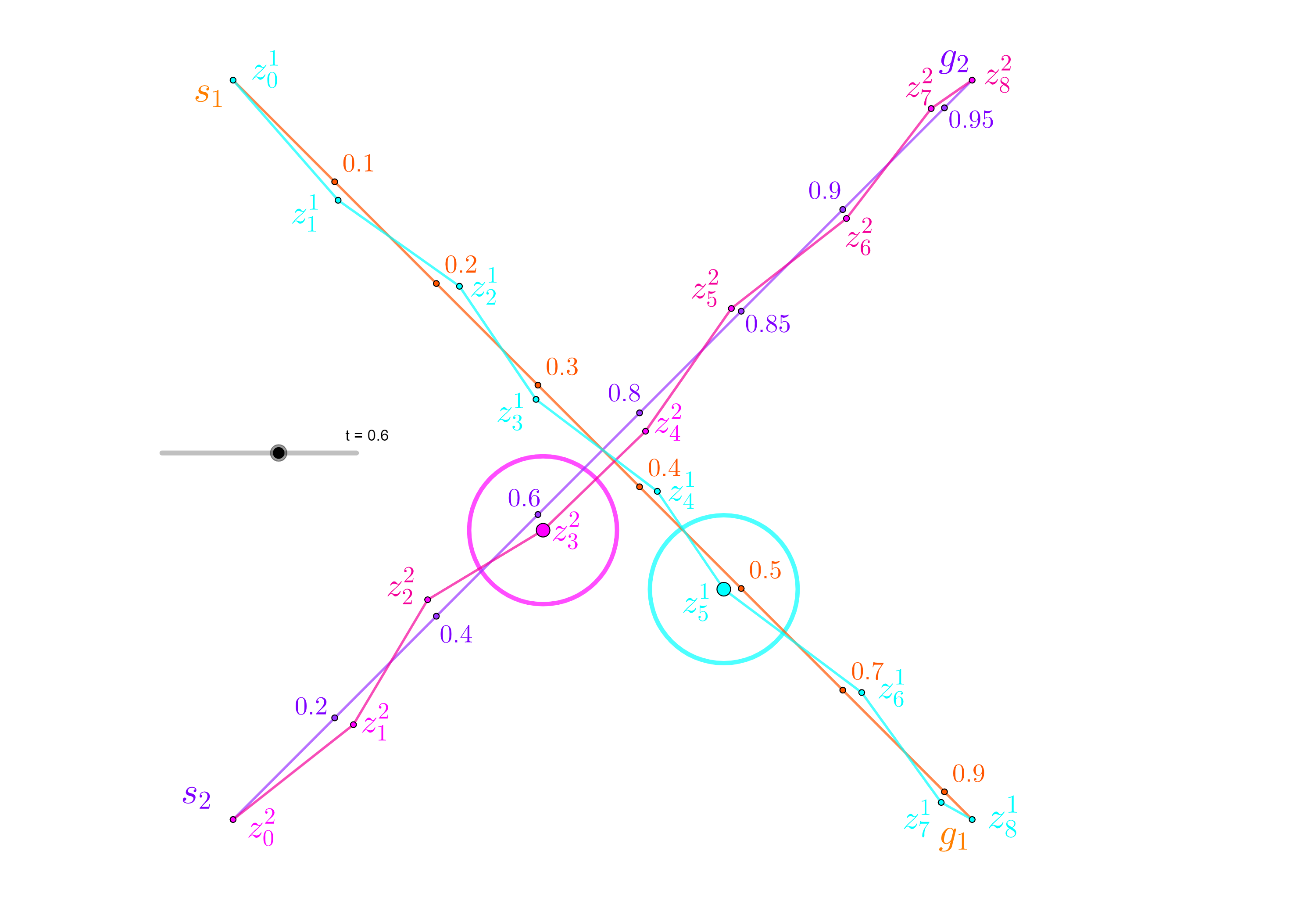}
    \caption{$V_8$}
  \end{subfigure}
  \begin{subfigure}[b]{0.95\linewidth}
  \centering
    \includegraphics[trim=5.5 1 8 1, clip, width=1\textwidth]{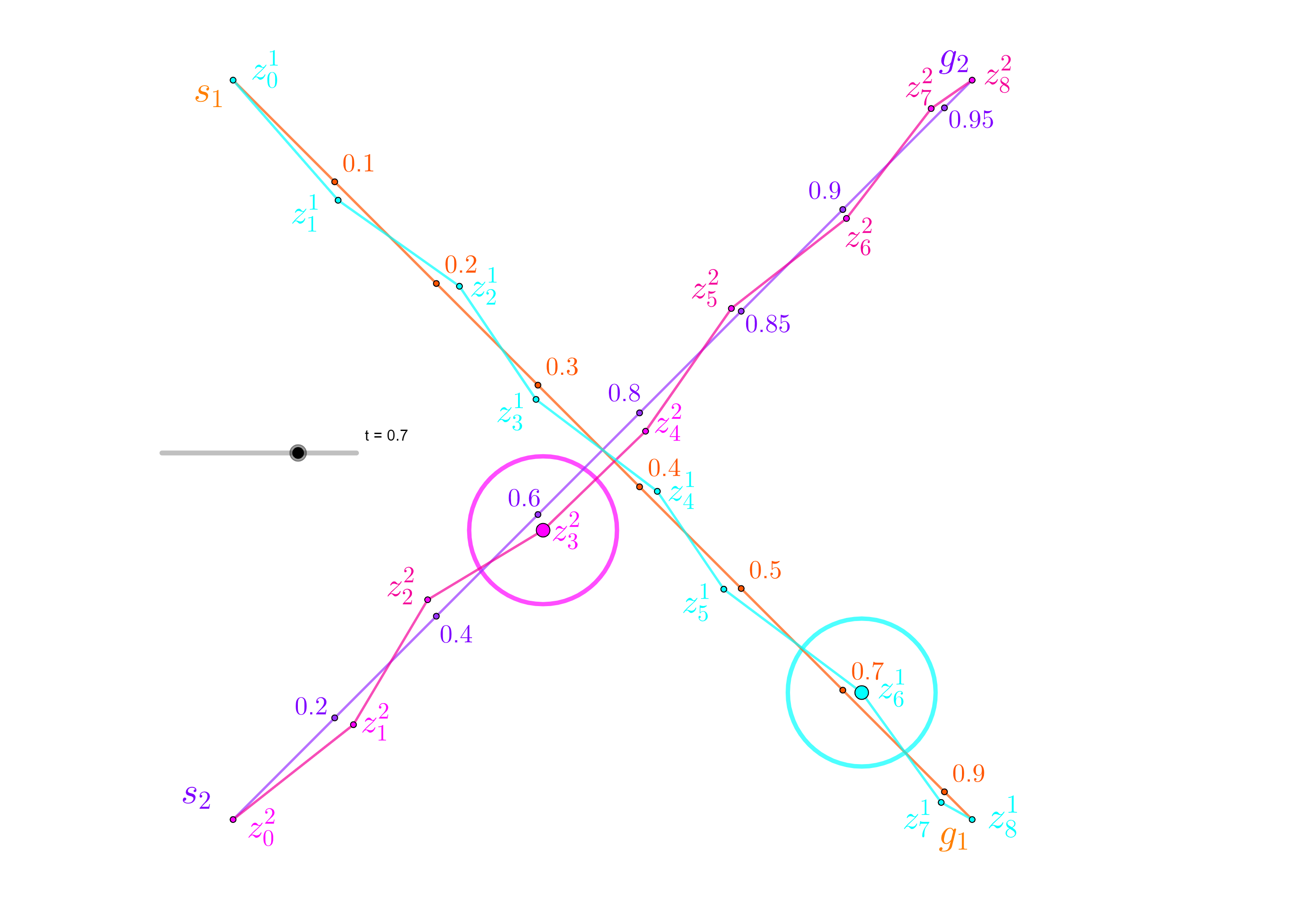}
    \caption{$V_9$}
  \end{subfigure}

  \caption{\sf{An illustration for the example in the. appendix. The robots are denoted as empty discs. Robot $1$ is denoted in cyan. Its reference path is the orange line, the path calculated for it is represented by the cyan path. Notice for each time stamp in $T_1$ there exist a point on the orange path and a the point from $Z_1$ next to it in cyan. Robot $2$ is denoted in pink. Its reference path is the purple line and the path calculated for it is represented by the pink path.}}
  \label{fig:example-fig}
\end{figure}

\end{document}